\RequirePackage{iftex}
\ifPDFTeX
  \pdfoutput=1
\fi
\documentclass[11pt]{article}

\usepackage[utf8]{inputenc}
\usepackage[T1]{fontenc}
\usepackage[english]{babel}
\usepackage{lmodern}
\usepackage{microtype}

\usepackage[right=1in, top=1in, bottom=1in, left=1in]{geometry}
\usepackage{setspace}
\usepackage{amsmath}
\usepackage{amssymb}
\usepackage{amsthm}
\usepackage{mathtools}
\usepackage{bbm}
\usepackage{bm}

\usepackage[dvipsnames]{xcolor}

\usepackage{graphicx}
\graphicspath{{src/figures/}}
\usepackage{tikz}
\usetikzlibrary{arrows.meta,positioning,calc}

\usepackage{booktabs}
\usepackage{array}
\usepackage{tabularx}

\usepackage{caption}
\usepackage{subcaption}

\usepackage[shortlabels]{enumitem}

\usepackage[authoryear,round]{natbib}
\usepackage[colorlinks=true, linkcolor=blue, citecolor=blue, hypertexnames=false]{hyperref}

\renewcommand{\arraystretch}{1.15}

\newcommand{\KL}{\mathrm{KL}}
\newcommand{\Ber}{\mathrm{Bern}}

\newtheorem{theorem}{Theorem}[section]
\newtheorem{proposition}[theorem]{Proposition}
\newtheorem{lemma}[theorem]{Lemma}
\newtheorem{corollary}[theorem]{Corollary}

\theoremstyle{definition}
\newtheorem{definition}[theorem]{Definition}

\theoremstyle{remark}
\newtheorem{remark}[theorem]{Remark}
\newtheorem{example}[theorem]{Example}

\title{Bet on Features: Anytime-Valid and Feature-Aware Auditing of Conditional Quantile Forecasters}

\author{\begingroup
\setlength{\tabcolsep}{6pt}
\renewcommand{\arraystretch}{0.88}
\begin{tabular}{cc}
\begin{tabular}[t]{@{}c@{}}
Ivane Antonov\thanks{Equal contribution.}\\[-0.35ex]
{\footnotesize Julius-Maximilians-Universit\"at W\"urzburg}\\[-0.55ex]
{\footnotesize \texttt{ivane.antonov@uni-wuerzburg.de}}
\end{tabular}
&
\begin{tabular}[t]{@{}c@{}}
Sohom Mukherjee\footnotemark[1]\\[-0.35ex]
{\footnotesize Julius-Maximilians-Universit\"at W\"urzburg}\\[-0.55ex]
{\footnotesize \texttt{sohom.mukherjee@uni-wuerzburg.de}}
\end{tabular}
\\[5.0ex]
\begin{tabular}[t]{@{}c@{}}
Richard Pibernik\\[-0.35ex]
{\footnotesize Julius-Maximilians-Universit\"at W\"urzburg}\\[-0.55ex]
{\footnotesize Zaragoza Logistics Center}\\[-0.55ex]
{\footnotesize \texttt{richard.pibernik@uni-wuerzburg.de}}
\end{tabular}
&
\begin{tabular}[t]{@{}c@{}}
Yo Joong Choe\\[-0.35ex]
{\footnotesize INSEAD}\\[-0.55ex]
{\footnotesize \texttt{yojoong.choe@insead.edu}}
\end{tabular}
\end{tabular}
\endgroup
}
\date{}

\begin{document}

\maketitle

\begin{center}
\textbf{Abstract}
\end{center}
\vspace{-0.6em}

\begingroup
\setstretch{1.0}
\footnotesize
\noindent
Black-box conditional quantile forecasts are widely used for
sequential decisions under asymmetric costs. Once deployed, such forecasters must be monitored continuously as data streams drift and regimes change; this invalidates standard, fixed-horizon backtests for calibration. Further, existing backtests do not take into account that the notion of calibration is, in fact, information-dependent: forecasts can look calibrated to an auditor with coarse information while being miscalibrated to an auditor with richer information. We develop a distribution-free and game-theoretic testing framework for
continuously auditing such forecasters with dependent,
potentially nonstationary observations.
We first formalize notions of conditional quantile calibration
relative to the contextual features available to the auditor, showing how coarser information can limit detection. We then identify detectable alternative classes and,
for contextual bets linear in the features, derive finite-time detection
guarantees without an i.i.d.~assumption.
The resulting evidence processes provide interpretable,
``feature-aware'' evidence for miscalibration.
We empirically validate these methods on simulated and real data, finding that
a popular time series forecaster (Chronos-2) is highly miscalibrated
w.r.t.~multiple relevant features.
\par
\endgroup

\section{Introduction}
\label{sec:intro}

\begin{figure*}[t]
    \centering
    \begin{tikzpicture}
        \node[anchor=north, inner sep=0pt] (legend) at (0, 0) {\includegraphics[width=\textwidth]{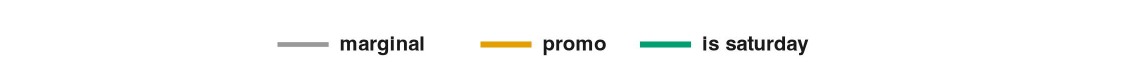}};
        \node[anchor=south, inner sep=0pt, yshift=-10pt] (plots) at (legend.north) {\begin{minipage}[t]{0.455\textwidth}
                \centering
                \includegraphics[width=\linewidth]{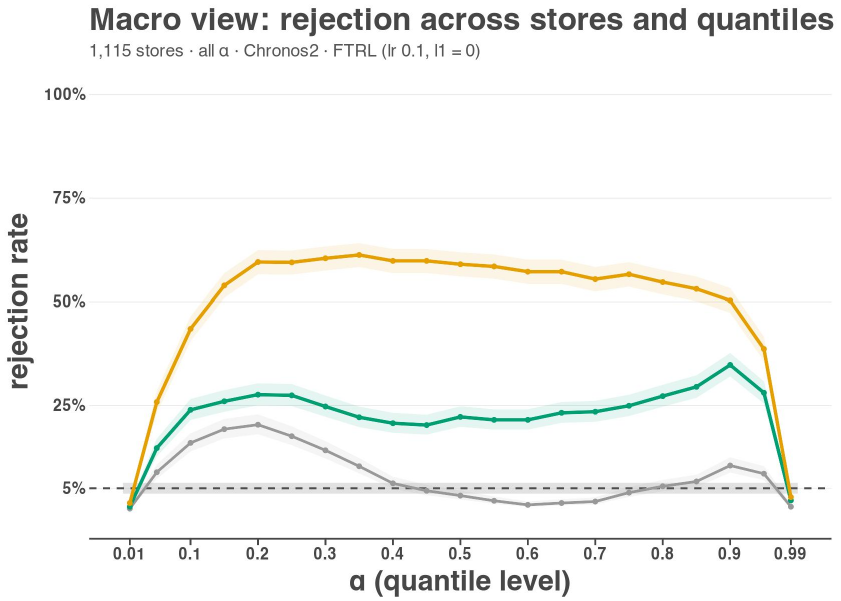}
            \end{minipage}\hspace{0.04\textwidth}\begin{minipage}[t]{0.455\textwidth}
                \centering
                \includegraphics[width=\linewidth]{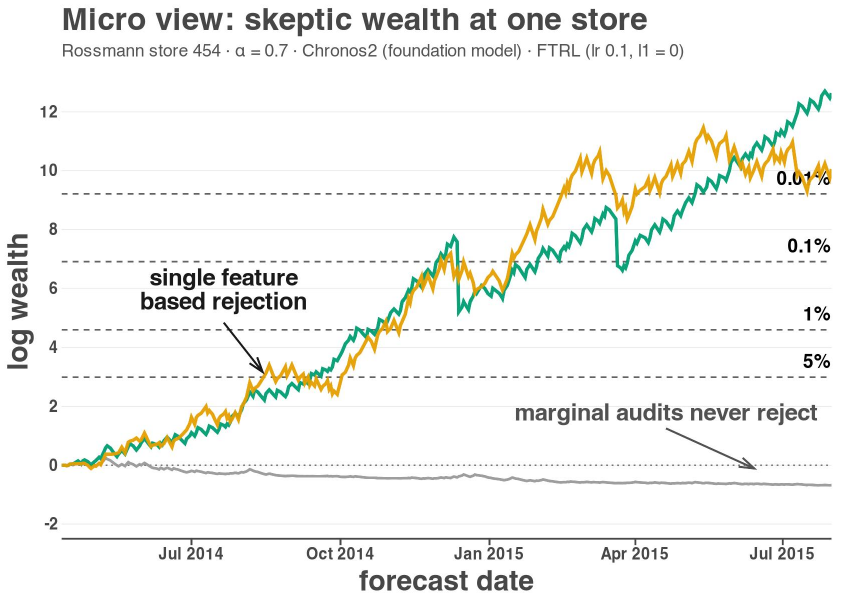}
            \end{minipage}};
    \end{tikzpicture}
    \caption{
Chronos-2 quantile forecasts on Rossmann store-sales data can pass marginal
audits while failing feature-aware audits. \textbf{Left:} rejection rate
across stores by quantile level. \textbf{Right:} wealth paths for one store at
$\alpha=0.7$. FTRL skeptics of Appendix~\ref{app:simulation_protocol}.
}
    \label{fig:intro-placeholder}
\end{figure*}

{Rigorous backtests of econometric forecasts must be sensitive to \emph{conditional} miscalibration. 
This was the requirement put forth by the seminal work of \citet{christoffersen1998evaluating}, who established that forecasting errors that are contextually predictable can nevertheless be invisible to the uninformed auditor, as the errors can average out marginally.
Conditional backtests are increasingly relevant, as more and more covariate-rich time series forecasting applications~\citep[e.g.,][]{yang2025llm} in the real world are handled by large, black-box models~\citep[e.g.,][]{ansari2025chronos,liu2025moirai}.}

Modern deployments add a second difficulty: forecasts need to be monitored continuously. 
Practitioners inspect calibration as outcomes arrive, stop once evidence accumulates, or intervene after a suspected distribution shift~\citep{hoga2023monitoring}. 
Fixed-horizon backtests are not designed for this workflow and lose their Type-I error guarantees under optional stopping. 
{This necessitates employing the \emph{anytime-valid} inference paradigm~\citep{ville1939etude,shafer2011test,ramdas2023game,grunwald2024safe}, which found its use in operational contexts where validity under continuous monitoring is critical~\citep[e.g.,][]{johari2021always,choe2024comparing,ebacktesting}.}

In particular, for decision-making problems under asymmetric cost, e.g., inventory management \citep{cao2019quantile}, public health response \citep{doms2018assessing}, and financial risk \citep{engle2004caviar}, the relevant action is often
based on the conditional quantile (as opposed to the conditional mean) based on past history \citep{koenker1978regression}.
Persistent bias in the quantile forecasts (too high or too low relative to the conditional law) translates to downstream decision errors, necessitating \emph{conditional quantile calibration}. 
In econometrics and risk management, testing the calibration of conditional quantiles is known as backtesting value-at-risk~\citep[VaR;][]{christoffersen2004backtesting}.

{
Existing backtests for conditional quantiles are usually valid in a fixed horizon and \emph{not} anytime-valid~\citep[e.g.,][]{christoffersen1998evaluating, engle2004caviar, berkowitz2011evaluating, nolde2017elicitability, hoga2023monitoring}. 
While some notable recent advances~\citep{casgrain2024sequential,ebacktesting} develop anytime-valid counterparts, they do not consider a practical evaluation scenario where the \emph{auditor} has only limited, yet potentially useful, information to the evaluation. 
The focus of this paper is to make the auditor's information (``features'') explicit in the notion of conditional quantile calibration, making the resulting audit \emph{feature-aware}.}
In real-world forecasting scenarios, the forecaster often has access to rich internal context, while an external auditor observes only a coarser view. For instance, a wholesaler auditing a retailer's demand forecasts may see aggregate
sales and past forecast errors, but not the local retailer's promotion calendar. 
These situations cannot be modelled by a single null hypothesis: the calibration property being tested,
and the alternatives against which the audit can have power, are determined by the information
available to the auditor.
Figure~\ref{fig:intro-placeholder} shows this gap on Rossmann store-sales data \citep{cukierski2015rossmann} using Chronos-2 \citep{ansari2025chronos} forecasts: marginal audits rarely reject at central quantile levels, while promotion- and Saturday-aware audits reveal feature-specific miscalibration and accumulate anytime-valid evidence. 

{ To formalize this idea, we build on the testing-by-betting framework for anytime-valid inference~\citep{ramdas2023game}, where the testing problem is cast as a sequence of hypothetical betting games against the forecaster.
In each round, the forecaster issues a quantile forecast, the outcome is observed, and the auditor records whether the outcome fell
below the forecast. Before seeing the next outcome, the auditor (``skeptic'') chooses a bet using only its own monitoring information. 
If the forecast is calibrated relative to that information, every legal betting strategy produces a nonnegative test martingale and thus an \emph{e-process}~\citep{ramdas2022testing,grunwald2024safe}, a nonnegative process that is expected to be at most one at arbitrary stopping times. 
The accumulated wealth in the game is valid evidence (e-process) against calibration even under data-dependent stopping, and the resulting sequential test (stop if wealth multiplies $20\times$) is, by Ville's inequality, valid (at level $\gamma=1/20$) under continuous monitoring~\citep[see, e.g.,][]{ramdas2025ebook}. }

The proposed construction uses only issued forecasts and realized outcomes~\citep{1999_dawid_prequential}; it does not require a parametric
model, stationarity, independence, or moment assumptions on the outcomes.
This information-indexed view separates validity from power. 
Coarser audits are safe, but they can be blind: when the forecaster is systematically wrong but in contexts only visible through unobserved features, the audit is anytime-valid but powerless. 
This means that the auditor's available information determines
both the calibration null being tested and the alternatives the audit can detect.
To gain power, the auditor must bet on predictable structure visible in its own information. 
We formalize this idea in the framework of \emph{contextual betting}:
the auditor comes with a predictable feature dictionary (such as calendar indicators, prices, promotions, or lagged forecasts), and an online learner adapts the betting direction from the observed stream. 
The learned weights are interpretable, identifying which features expose conditional miscalibration.
Theoretically, we will show that the expected rejection time of the audit is finite under persistent miscalibration, without having to assume i.i.d.~outcomes or errors.  

\subsection{Related work and contributions}

Some recent work in anytime-valid inference serve as key building blocks, although there are clear distinctions in terms of goals, methods, and/or assumptions. 
\citet{howard2022sequential} construct time-uniform confidence sequences for population quantiles from i.i.d. observations, while
\citet{mineiro2023timeuniform} build confidence bands for
CDFs under nonstationarity; both are targeted at estimation rather than conditional forecast audits with feature-specific notions of evidence. 
\citet{henzi2022sequentialcalibration} develop sequentially valid tests
for probabilistic forecast calibration, focusing on distributional forecasts rather than a single quantile with associated asymmetric costs. 
{\citet{henzi2021scoringrule,choe2024comparing} develop anytime-valid forecast comparison methods based on scoring rules (though not specific to calibration); \citet{choe2026combining} develop general combination strategies for e-processes across different information sets but for the same null (here, the null hypotheses differ for auditors with different information sets).}

Methodologically, \citet{casgrain2024sequential} provide a general framework for identifiable functionals, and our audit builds on their test martingales for quantiles. 
The key distinction is that the null considered in this paper is indexed by the auditor's
monitoring information, leading to a study of feature-aware validity and non-asymptotic power, as opposed to their asymptotic power-one result. 
The e-backtesting framework of \citet{ebacktesting} is the other closely related work, as their e-value for VaR testing coincides with an (extreme) betting strategy in our framework. 
The main differences are that e-backtesting is also not specific to the auditor's information set, and it studies asymptotic power in i.i.d.~settings, whereas we study information-indexed calibration and learn contextual bets through online convex optimization (OCO). 
Further related work appears in Appendix~\ref{sec:related}. 

The major contributions of this paper are summarized as follows:
\begin{enumerate}[leftmargin=*]
    \item \textbf{Information-indexed, anytime-valid calibration audits.}
    We formulate continuous auditing of black-box conditional quantile forecasts as a sequential
    testing-by-betting game indexed by the auditor's monitoring information
    (Section~\ref{sec:preliminaries}). Predictable no-bankruptcy bets generate test martingales, hence
    e-processes, that are valid under arbitrary non-i.i.d.~streams and optional stopping (Corollary \ref{cor:anytime-valid-monitoring}). We prove a
    hierarchy of calibration nulls, show that validity transfers from coarser to richer nulls, and show
    that coarser valid audits can nevertheless be powerless against richer feature-conditional
    violations (Corr.~\ref{cor:validity-transfer}, Prop.~\ref{prop:coarser-tests-can-miss}).

    \item \textbf{Contextual (``feature-aware'') betting with finite-time power.}
    We develop contextual betting strategies that learn linear bets over a predictable feature
    dictionary available to the auditor (Section~\ref{sec:power_stopping_time}). For non-i.i.d.~alternatives with
    persistent feature-aligned predictable edge, a regret-to-power theorem converts any pathwise
    OCO regret bound into finite-time detection and stopping-time guarantees (Theorem \ref{thm:power-linear-full-feature}). 
    A more general alternative allowing time-varying or intermittent violations is analyzed in Appendix~\ref{apx:missing_results}. 

    \item \textbf{Empirical evidence for feature-specific miscalibration.}
    We validate our proposed audit in three settings (Sec.~\ref{sec:experiments} and Appx.~\ref{app:experiments}). On a controlled two-state
    Markov hit process, we exhibit the marginal-versus-conditional gap. Under a calibrated Negative-Binomial oracle, empirical rejection rates respect the Ville-prescribed test
    level across quantiles (Type-I validity). Replacing the oracle with Chronos-2 forecasts on synthetic and Rossmann
    store-sales data, feature-aware skeptics detect conditional miscalibration that marginal audits miss. A thorough comparison against the two-sided GREM-90 e-backtesting baseline verifies the strength of contextually learned betting on a wide range of quantiles.
\end{enumerate}

\section{Information-aware
anytime audits}
\label{sec:preliminaries}

\textbf{Operational audit with asymmetric information.}
We formalize conditional quantile forecast calibration auditing as a sequential testing-by-betting game. A central theme of our work is \emph{information asymmetry}. In operational forecasting, the forecaster may have access to rich
internal context (e.g., promotions, store identifiers), while the auditor may observe only a
coarser view. 
A coarser audit may fail
to detect miscalibration that is visible only through the presence of a missing feature. This
\textit{feature-aware} notion of evidence is the organizing
principle of the paper.

\textbf{Conditional quantile forecast calibration game.}
Let $(\Omega,\mathcal F,\mathbb F)$ be a filtered measurable space, with $\mathbb F=(\mathcal F_t)_{t\ge0}$. All laws considered below are probability measures on $(\Omega,\mathcal F)$. We index the filtration so that $\mathcal F_{t-1}$ is the richest pre-outcome information available before observing the outcome $Y_t$. For rounds $t=1,2,\ldots$, we have the following sequence of events. The covariates $C_t$ which is $\mathcal F_{t-1}$-measurable is revealed. For fixed $\alpha\in(0,1)$, the black-box forecaster issues a $\mathcal F_{t-1}$-measurable point forecast $\hat q_t \in \mathbb R$. The forecast is intended as a conditional $\alpha$-quantile forecast for $Y_t$. Finally, the outcome $Y_t$ is revealed, which is $\mathcal F_t$-measurable. 

Moreover, we will use a predictable feature dictionary
\(
    \phi_t\in\mathbb R^d, \mathcal F_{t-1}\text{-measurable},
\)
to denote the full collection of audit-relevant pre-outcome features (or analogously context), i.e., including $C_t$ and possibly further information like past forecasts and hit sequences.
A more nuanced formulation of the conditional quantile calibration game involving the \textit{forecaster}, the \textit{skeptic} (game-theoretic terminology for the auditor), and \textit{nature} can be found in Appendix \ref{apx:game_protocol}.
To define the null hypothesis for conditional quantile calibration, we first set up some notations.
Define the uncentered hit and centered hit by
\(
    B_t:=\mathbf 1\{Y_t\le \widehat q_t\},
\) and
\(    Z_t:=B_t-\alpha.
\)
Since $\hat q_t$ is $\mathcal F_{t-1}$-measurable and $Y_t$ is $\mathcal F_t$-measurable, $Z_t$ is $\mathcal F_t$-measurable and takes values in $\{-\alpha,1-\alpha\}$.

\textbf{Monitoring information structure.}
What is crucial to our further discussion is the different information sets available to the auditor. A monitoring information structure is a sequence
$\mathbb H=(\mathcal H_t)_{t\ge0}$ such that, before the outcome at time $t$ is
observed,
\(
    \mathcal H_{t-1}\subseteq \mathcal F_{t-1},
\)
and after the outcome is observed the hit $Z_t$ is included in the monitored
information. 
We consider three monitoring levels. First, is the marginal information denoted by $\mathcal H_{t-1}^{\rm marg}$. This could include, for example, the past hit record or some aggregate thereof. Second, we have the single-feature (or scalar-feature) information denoted by $\mathcal H_{t-1}^{w}$. This could include, for example, a single coordinate $\phi_{t,j}$ of the covariate. {(For simplicity, we will sometimes use the notation $w \in \phi$ to indicate a single feature within the full feature dictionary.)} Finally, we have the full-feature monitoring information denoted by $\mathcal H_{t-1}^{\phi}$. By construction,
\(
    \mathcal H_{t-1}^{\rm marg}
    \subseteq
    \mathcal H_{t-1}^{w}
    \subseteq
    \mathcal H_{t-1}^{\phi}
    \subseteq
    \mathcal F_{t-1},    
\)
for every $t$.

\begin{definition}[$\mathbb H$-conditional exact-hit calibration null]
\label{def:strongnull}
For any monitoring information structure $\mathbb H$, define the
information-indexed exact-hit calibration null hypothesis as the composite class
\begin{equation}\label{eq:cond_exact_null}
\mathcal P_0(\mathbb H)
:=\bigl\{P:
P(Y_t\le\widehat q_t\mid\mathcal H_{t-1})=\alpha,\, P\text{-a.s., } \forall t\ge1
\bigr\}.
\end{equation}
Equivalently, $P\in\mathcal P_0(\mathbb H)$ iff
$\mathbb E_P[Z_t\mid\mathcal H_{t-1}]=0$,
$P$-a.s., for all $t \geq 1$.
\end{definition}

The exact-hit null requires exact conditional coverage.
For general outcome distributions, conditional quantile calibration
is expressed by the weaker pair of inequalities
\(
P(Y_t<\widehat q_t\mid\mathcal H_{t-1})
\le\alpha\le
P(Y_t\le\widehat q_t\mid\mathcal H_{t-1}),
\, P\text{-a.s.},
\)
for every $t\ge1$.
These inequalities coincide with exact-hit calibration when
$P(Y_t=\widehat q_t\mid\mathcal H_{t-1})=0$,
$P$-a.s. for every $t$.
With atoms at the forecast, a correct quantile need not satisfy
the exact-hit equality. We develop the exact-hit formulation below for a clear exposition, and provide a straightforward extension to the general quantile null in
Appendix~\ref{app:discrete-quantiles}.

The equivalence (at the end of Def. \ref{def:strongnull}) is due to the bounded martingale difference sequence (MDS) property of $(Z_t)_{t\ge1}$ with respect to $\mathbb H$, and formally proved in Proposition \ref{prop:bounded-monitored-mds}.
Thus $\mathcal P_0^{\rm marg}
    :=
    \mathcal P_0(\mathbb H^{\rm marg})$ is marginal sequential calibration,
$\mathcal P_0^{w}
    :=
    \mathcal P_0(\mathbb H^{w})$ is calibration with respect to one scalar feature, and $\mathcal P_0^{\phi}
    :=
    \mathcal P_0(\mathbb H^{\phi})$ is calibration with respect to the full feature
dictionary. The nulls follow a hierarchy as outlined in the following Proposition.

\begin{proposition}[Null hierarchy]
\label{prop:null-hierarchy}
Let $\mathbb H$ and $\mathbb G$ be two monitoring information structures such
that
\(
    \mathcal H_{t-1}\subseteq \mathcal G_{t-1},
\)
for every $t\ge1$. Then
\(
    \mathcal P_0(\mathbb G)
    \subseteq
    \mathcal P_0(\mathbb H).
\)
Consequently,
\(
    \mathcal P_0^\phi
    \subseteq
    \mathcal P_0^w
    \subseteq
    \mathcal P_0^{\rm marg}.
\)
Thus full-feature calibration is the strongest requirement and marginal
calibration is the weakest requirement.
\end{proposition}
In particular, if we reject the single-feature calibration null $\mathcal{P}_0^w$, w.r.t.~\emph{any} feature, then we readily reject the full-feature calibration null $\mathcal{P}_0^\phi$ (see Figure~\ref{fig:null-hierarchy}). 
Moreover, rejecting $\mathcal{P}_0^w$ for a particular feature would give us strictly more information, in that we additionally know conditional on which feature the forecaster is miscalibrated. 
We can view the testing problem studied in existing work~\citep[e.g.,][]{casgrain2024sequential,ebacktesting} as the case where $\mathcal{H}_t \equiv \mathcal{F}_t,\, \forall t$. 
The corresponding null hypothesis is at least as small as $\mathcal{P}_0^\phi$, and thus it is an easier testing problem for which there are more possibilities for rejection. 

\begin{figure*}[t]
    \centering
    \begin{subfigure}[t]{0.475\textwidth}
        \centering
        \includegraphics[width=\linewidth]{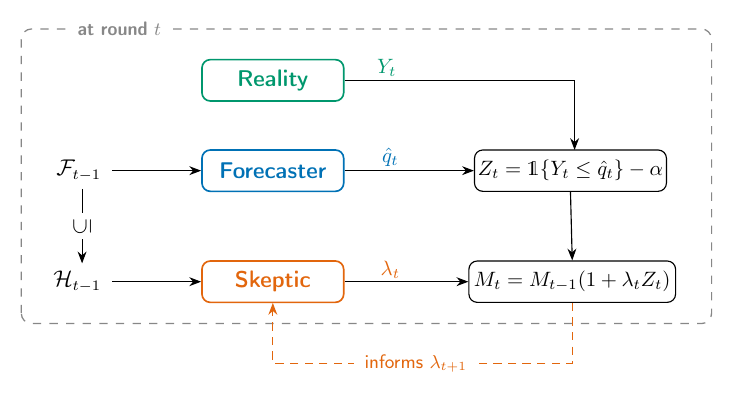}
        \caption{One conditional quantile calibration game round (Def.~\ref{def:filtration}):
Forecaster issues an $\mathcal F_{t-1}$-measurable forecast
$\widehat q_t$;
Skeptic, with monitoring information
$\mathcal H_{t-1}\subseteq\mathcal F_{t-1}$,
chooses an $\mathcal H_{t-1}$-measurable bet
$\lambda_t\in\Lambda_\alpha$;
Reality reveals outcome $Y_t$; $Z_t=\mathbf 1\{Y_t\le\widehat q_t\}-\alpha$
is recorded in $\mathcal H_t$, wealth updates as
$M_t=M_{t-1}(1+\lambda_tZ_t)$.}
        \label{fig:game-protocol}
    \end{subfigure}
    \hfill
    \begin{subfigure}[t]{0.475\textwidth}
        \centering
        \includegraphics[width=\linewidth]{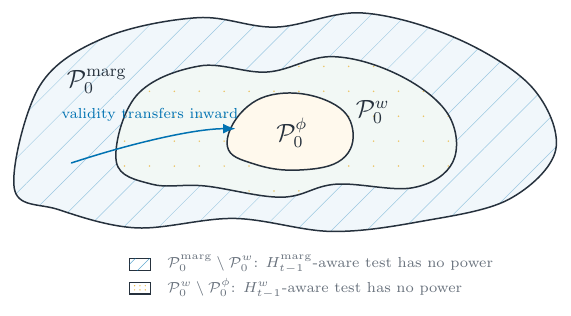}
        \caption{Hierarchy of calibration nulls indexed by the monitoring sub-filtration (Proposition~\ref{prop:null-hierarchy}). Refining the monitoring information structure shrinks the null: $\mathcal P_0^\phi \subseteq \mathcal P_0^w \subseteq \mathcal P_0^{\rm marg}$. Coarser information tests might fail to gain power (Proposition \ref{prop:coarser-tests-can-miss}).}
        \label{fig:null-hierarchy}
    \end{subfigure}
    \caption{Left: game round with information asymmetry \& feature-aware bets. Right: induced hierarchy of nulls.}
    \label{fig:game-and-nulls}
\end{figure*}
\textbf{Test supermartingales and e-processes.}
We now recall the notions of sequential testing needed to convert the bounded MDS property above into anytime-valid evidence. The following definitions are standard in the modern e-value literature 
\citep{vovk2021evalues,ramdas2022testing,ramdas2023game,grunwald2024safe,waudbysmith2024testing,ramdas2025ebook}.

\begin{definition}[Test supermartingales and e-processes]
\label{def:eprocess}
\label{def:test-supermartingale-eprocess}
A nonnegative, $\mathbb H$-adapted process $(S_t)_{t\ge0}$ is a test supermartingale for $\mathcal P_0(\mathbb H)$ if $S_0\le1$ and, for every $P\in\mathcal P_0(\mathbb H)$, $\mathbb E_P[S_t\mid\mathcal H_{t-1}]\le S_{t-1}$, $P$-a.s., for all $t\ge1$. 
An e-process for $\mathcal P_0(\mathbb H)$ is a nonnegative, $\mathbb H$-adapted process $(M_t)_{t\ge0}$ such that, for every $P\in\mathcal P_0(\mathbb H)$, there exists a test supermartingale $(S_t^P)_{t\ge0}$ with $M_t\le S_t^P$ $P$-a.s. for all $t\ge0$. Every test supermartingale is therefore an e-process.
\end{definition}

\begin{definition}[Skeptic's wealth process and betting domain]
\label{def:skeptic-wealth-process}
An
$\mathbb H$-predictable betting strategy is a sequence
$(\lambda_t)_{t\ge1}$ such that $\lambda_t$ is
$\mathcal H_{t-1}$-measurable and takes values in $\Lambda_\alpha=[-1/(1-\alpha),\,1/\alpha]$. The wealth process for a skeptic with such bets is \(M_t:=\prod_{i=1}^t(1+\lambda_i Z_i)\), where $M_0:=1$.
\end{definition}

The domain $\Lambda_\alpha$ is the maximal no-bankruptcy interval, i.e., the largest interval ensuring \(1+\lambda_tZ_t\ge0\). Since $Z_t\in\{-\alpha,1-\alpha\}$, nonnegativity is equivalent to $1+\lambda(1-\alpha)\ge0$ and $1-\alpha\lambda\ge0$, which simplifies to $\Lambda_\alpha$. Under the null, these multipliers have conditional mean one: for every predictable $\lambda_t$ taking values in $\Lambda_\alpha$ and every $P\in\mathcal P_0(\mathbb H)$, $\mathbb E_P[1+\lambda_tZ_t\mid\mathcal H_{t-1}]=1+\lambda_t\mathbb E_P[Z_t\mid\mathcal H_{t-1}]=1$. Hence $E_t:=1+\lambda_tZ_t$ is a conditional e-variable (Definition \ref{def:conditional-evariables}), and the product wealth process is a test martingale under every null law. 
We formally state this well-known construction~\citep[e.g.,][]{waudbysmith2024testing} adapted to this problem as a Theorem below; by Ville's inequality~\citep{ville1939etude}, we immediately have an anytime-valid test. An anytime-valid extension for the general quantile null can be obtained
by averaging two one-sided wealth processes; we defer the details to Appendix~\ref{app:discrete-quantiles} due to space constraints.

\begin{theorem}[Validity of predictable betting under a monitoring null]
\label{thm:validity-monitoring}
Let $(\lambda_t)_{t\ge1}$ be any $\mathbb H$-predictable betting strategy taking
values in $\Lambda_\alpha$, define $(M_t)$ as in Def. \ref{def:skeptic-wealth-process}. If $P\in\mathcal P_0(\mathbb H)$, then
$(M_t)_{t\ge0}$ is a nonnegative $P$-martingale with respect to the monitoring
filtration $\mathbb H$. Consequently, $(M_t)$ is a test martingale, and hence an
e-process, for $\mathcal P_0(\mathbb H)$.
\end{theorem}

\begin{corollary}[Ville's inequality]
\label{cor:anytime-valid-monitoring}
Under the assumptions of Theorem~\ref{thm:validity-monitoring}, for every
$P\in\mathcal P_0(\mathbb H)$ and every $c>0$,
\(
    P\!\left(\sup_{t\ge0}M_t\ge c\right)
    \le
    \frac1c.
\)
Equivalently, for any $\gamma\in(0,1)$, the rejection time $\tau_\gamma:=\inf\{t\ge1:M_t\ge1/\gamma\}$, with $\inf\emptyset:=\infty$, satisfies
\(
    \sup_{P\in\mathcal P_0(\mathbb H)}
    P(\tau_\gamma<\infty)
    \le
    \gamma .
\)
\end{corollary}

\textbf{Validity transfers inward while power is lost under coarsening.}
For our problem, we are particularly interested in how the validity of tests can transfer from coarser (larger) to richer (smaller) nulls. Intuitively, this is true because Type I error (or equivalently, e-process validity) holding uniformly over a null set also holds uniformly over its subset. This is formalized in the following Corollary.

\begin{corollary}[Validity transfers from coarser to richer nulls]
\label{cor:validity-transfer}
Suppose that two monitoring filtrations
$\mathbb H=(\mathcal H_t)_{t\ge0}$ and
$\mathbb G=(\mathcal G_t)_{t\ge0}$ satisfy
\(
    \mathcal H_{t-1}\subseteq\mathcal G_{t-1}
\)
for every $t$.
Then every $\mathbb H$-predictable betting process that is valid for the larger
null $\mathcal P_0(\mathbb H)$ is automatically valid for the smaller, richer
null $\mathcal P_0(\mathbb G)$.
In particular,
\(
    \text{marginal-valid tests are valid under }\mathcal P_0^w
    \text{ and }\mathcal P_0^\phi,
\)
and
\(
    \text{one-feature-valid tests are valid under }\mathcal P_0^\phi.
\)
\end{corollary}

While the previous corollary shows validity, it does not guarantee anything about the power of the coarser information test against violations of the richer null. A coarser test can remain perfectly valid under a richer null while being blind to its violations. The next proposition formalizes this point.

\begin{proposition}[Coarser valid tests can miss richer-null violations]
\label{prop:coarser-tests-can-miss}
Let $\mathbb H$ and $\mathbb G$ satisfy
$\mathcal H_{t-1}\subseteq\mathcal G_{t-1}$ for every $t$. Suppose a law $Q$
satisfies
\(
    Q\in\mathcal P_0(\mathbb H)
    \setminus
    \mathcal P_0(\mathbb G).
\)
Then $Q$ violates the richer $\mathbb G$-conditional calibration null, but
every $\mathbb H$-predictable no-bankruptcy wealth process remains a
nonnegative martingale under $Q$. Consequently, for every
$\gamma\in(0,1)$,
\(
    Q(\tau_\gamma<\infty)\le\gamma,
\)
where $\tau_\gamma=\inf\{t\ge1:M_t\ge1/\gamma\}$. Moreover, for every finite
horizon $T$,
\[
    \mathbb E_Q[\log M_T]\le0,
\]
i.e., a coarser auditor has no positive systematic log-growth under $Q$, even though $Q$ violates calibration at the richer information level.
\end{proposition}

\section{Feature-aware power and stopping times}
\label{sec:power_stopping_time}

The preceding section established the validity aspect of the audit: for any
monitoring information structure $\mathbb H$, every $\mathbb H$-predictable
no-bankruptcy betting strategy yields an anytime-valid e-process for
$\mathcal P_0(\mathbb H)$, with no distributional assumptions. We also noted that power of a test is information-sensitive: as Prop.~\ref{prop:coarser-tests-can-miss} shows, a skeptic using a
coarser information set can remain perfectly valid while having no positive
systematic log-growth against violations that are visible only after conditioning
on richer features. Equivalently, wealth grows only when
the chosen bets align with conditional deviations visible through the skeptic's information.

This section characterizes alternatives when a full-feature skeptic has finite-time power. An
alternative law $Q$ may be nonstationary and history-dependent; what matters is
whether the full-feature conditional miscalibration
\(
    m_t^\phi
    :=
    \mathbb E_Q[Z_t\mid\mathcal H_{t-1}^\phi]
\)
has a component that can be exploited by predictable feature-aware bets. If the
conditional hit probability were known, the skeptic could use the Kelly-optimal
stake that maximizes conditional expected log-growth. In practice this
probability is unknown, and the relevant miscalibration may be a feature-specific
pattern rather than a scalar bias. We therefore consider contextual bets of the
form
\(
    \lambda_t(\theta)=\langle\theta,\phi_t\rangle ,
\)
and use online learning to adapt $\theta_t$ from the observed stream. Thus the
full-feature skeptic can use the entire predictable feature dictionary to learn
profitable betting directions, while the learned weights also serve as
interpretable diagnostics for which features expose the forecaster's
conditional miscalibration.

\textbf{Alternatives induced by the hierarchy.}
We first describe how alternatives are induced by the hierarchy of monitoring
information. 
Fix an ambient model class $\mathcal P$ of probability laws on $(\Omega,\mathcal F)$.
For any monitoring information $\mathbb H$, define the corresponding alternative
as
\(
    \mathcal Q(\mathbb H)
    :=
    \mathcal P\setminus \mathcal P_0(\mathbb H).
\)
In particular,
\(
    \mathcal Q^{\rm marg}
    :=
    \mathcal P\setminus\mathcal P_0^{\rm marg},
    \mathcal Q^{w}
    :=
    \mathcal P\setminus\mathcal P_0^{w},
    \mathcal Q^{\phi}
    :=
    \mathcal P\setminus\mathcal P_0^{\phi}.
\)
Because the nulls are nested, the alternatives are nested in the reverse
direction:
\(
    \mathcal Q^{\rm marg}
    \subseteq
    \mathcal Q^{w}
    \subseteq
    \mathcal Q^{\phi}.
\)
Hence full-feature alternatives are the broadest: a forecaster may violate
full-feature calibration while still satisfying every coarser calibration null
available to a weaker auditor.

\textbf{Conditional means at different information levels.}
For any law $Q$ and any monitoring information $\mathbb H$, write
\(
    m_t^{\mathbb H}(Q)
    :=
    \mathbb E_Q[Z_t\mid\mathcal H_{t-1}].
\)
When the law is clear from context, we suppress $Q$. For the three canonical
levels, write
\(
    m_t^{\rm marg}
    :=
    \mathbb E_Q[Z_t\mid\mathcal H_{t-1}^{\rm marg}],
\)
\(
    m_t^{w}
    :=
    \mathbb E_Q[Z_t\mid\mathcal H_{t-1}^{w}],
\)
\(
    m_t^{\phi}
    :=
    \mathbb E_Q[Z_t\mid\mathcal H_{t-1}^{\phi}].
\)

For every law $Q\in\mathcal P$ and
$a\in\{{\rm marg},w,\phi\}$, we have $Q\in\mathcal Q^a$
if and only if there exists an integer $t\ge1$ such that
$Q(m_t^a\ne0)>0$. Moreover, by the tower property,
\(
    m_t^{\rm marg}
    =
    \mathbb E_Q[m_t^w\mid\mathcal H_{t-1}^{\rm marg}],
\)
\(
    m_t^w
    =
    \mathbb E_Q[m_t^\phi\mid\mathcal H_{t-1}^{w}].
\)
These identities quantify information loss: coarser skeptics observe only
conditional averages of the miscalibration visible at richer information levels.
The remainder of the section develops contextual betting strategies that exploit
the full-feature signal $m_t^\phi$ by learning feature-aligned bets online.

\paragraph{Full-feature contextual bets.}
Assume the full-feature vector $\phi_t\in\mathbb R^d$ is $\mathcal H_{t-1}^\phi$-measurable and satisfies $\|\phi_t\|_2\le R$ a.s. for all $t$, for a known constant $R>0$. Define the feasible set $K:=\{\theta\in\mathbb R^d:\|\theta\|_2\le 1/(2R)\}$. For $\theta\in K$, define the contextual betting fraction $\lambda_t(\theta):=\langle\theta,\phi_t\rangle$.
Since $|Z_t|\le1$ and $\|\phi_t\|_2\le R$, Cauchy--Schwarz yields
\(
|\lambda_t(\theta)Z_t|
\le \|\theta\|_2\|\phi_t\|_2|Z_t|
\le \frac12.
\)
Hence $1+\lambda_t(\theta)Z_t\ge 1/2>0$, so every predictable sequence $\theta_t\in K$ induces the valid no-bankruptcy wealth process
\[
M_t^{\rm ctx}
:=
\prod_{i=1}^t
\left(1+\langle\theta_i,\phi_i\rangle Z_i\right).
\]
For a fixed comparator $\theta\in K$, define the one-step loss $g_t(\theta):=-\log(1+\langle\theta,\phi_t\rangle Z_t)$ and the full-feature quadratic edge
\(
b_t^Q(\theta)
:=
\langle\theta,\phi_t\rangle m_t^\phi
-
\langle\theta,\phi_t\rangle^2.
\)
Next, we define the linear predictable-edge alternative $\mathcal Q_{\kappa,t_0}^{\phi,\operatorname{lin}}(\Phi)$.
For this alternative class, the full-feature miscalibration is exploitable by a linear betting direction in the dictionary $\Phi$. The skeptic can learn linear bets of the form $\lambda_t(\theta):=\langle\theta,\phi_t\rangle$ via OCO to obtain finite-time power and stopping-time bounds, against this alternative class.
We state the uniform-edge case in the main text because it yields a clean stopping-time bound
for persistent, linearly exploitable miscalibration
in Theorem~\ref{thm:power-linear-full-feature}.
A fixed comparator may still produce varying bets as the features change.
The more general cumulative-edge regret-to-power theorem
allows zero or negative one-step edges when the accumulated edge grows sufficiently relative to regret
(see Thm.~\ref{thm:full-feature-regret-to-power} and Remark~\ref{remark:connection_to_uniform_edge_cumulative}). Appendix~\ref{app:proofs} provides illustrative examples for our alternatives.

\begin{definition}[Linear full-feature predictable-edge alternative]
\label{def:linear-full-feature-alternative}
Fix
\(
0<\kappa\le\kappa_{\max}(\alpha)
:=\frac14\max\{\alpha^2,(1-\alpha)^2\},
\)
$t_0\in\mathbb N_{\ge1}$, and a predictable feature sequence
$\Phi=(\phi_t)_{t\ge1}$.
The linear full-feature predictable-edge alternative consists
of laws admitting one fixed comparator whose predictable edge
is at least $\kappa$ at every time from $t_0$ onward:
\begin{equation}
\label{eq:uniform_edge}
\begin{aligned}
&\mathcal Q_{\kappa,t_0}^{\phi,\operatorname{lin}}(\Phi)
\quad:=\bigl\{\,Q\in\mathcal P:
    \exists\,\theta^\star\in K\ \text{such that} \quad b_t^Q(\theta^\star)\ge\kappa
    \quad Q\text{-a.s.}\quad
    \forall\,t\in\mathbb N_{\ge t_0}\,\bigr\}.
\end{aligned}
\end{equation}
\end{definition}

\paragraph{Regret-to-power for full-feature contextual betting.}

Let an online learner output $\mathcal H_{t-1}^\phi$-measurable vectors $\theta_t\in K$. With $M_t^{\rm ctx}$ and $g_t$ defined as before, we can define the regret against a fixed comparator $\theta\in K$ by $R_T^{\rm ctx}(\theta):=\sum_{t=1}^T\{g_t(\theta_t)-g_t(\theta)\}$. We have the following result.



\begin{theorem}[Power against the linear full-feature alternative]
\label{thm:power-linear-full-feature}
Assume that the online learner satisfies the pathwise regret bound
$R_T^{\rm ctx}(\theta)\le r_T$ for every $\theta\in K$ and every
integer $T\ge1$, where $(r_T)_{T\ge1}$ is deterministic.
Fix $\gamma\in(0,1)$.
Let $Q\in\mathcal Q_{\kappa,t_0}^{\phi,\operatorname{lin}}(\Phi)$
with $0<\kappa\le\kappa_{\max}(\alpha)$.
Set
\(
    A_{\gamma,t_0}:=\log(1/\gamma)+2(t_0-1),
\)
and define
\[
\begin{aligned}
n_\gamma^{\rm ctx}(r):=\inf\bigl\{&
    n\in\mathbb N_{\ge t_0}:
    &A_{\gamma,t_0}+r_T\le\kappa T/2
    \quad\forall\,T\in\mathbb N_{\ge n}\bigr\},
\end{aligned}
\]
with $\inf\emptyset:=\infty$.
If $n_\gamma^{\rm ctx}(r)<\infty$, then for every integer
$T\ge n_\gamma^{\rm ctx}(r)$,
\(
    Q(\tau_\gamma^{\rm ctx}>T)
    \le \exp(-2\kappa^2T).
\)
Consequently,
\begin{gather*}
    Q\!\left(\tau_\gamma^{\rm ctx}<\infty\right)=1
    \tag*{\text{\textbf{\textit{(Power)}}}}
    \\[0.4em]
    \mathbb E_Q\!\left[\tau_\gamma^{\rm ctx}\right]
    \le n_\gamma^{\rm ctx}(r)+\frac{1}{2\kappa^2}
    \tag*{\text{\textbf{\textit{(Stopping time)}}}}
\end{gather*}
\end{theorem}

\begin{remark}[Interpreting the certification time $n_\gamma^{\rm ctx}(r)$]
\label{rem:certification-time}
The quantity $n_\gamma^{\rm ctx}(r)$ is a deterministic certification time, which is different form the stopping time: it is the first horizon after which the linear predictable
edge guaranteed by the alternative dominates the Ville threshold, the pre-$t_0$
penalty, and the learner's regret. 
Thus
\(
    n_\gamma^{\rm ctx}(r)<\infty
    \quad\Longleftrightarrow\quad
    \exists n\ge t_0
    \text{ such that }
    \sup_{T\ge n}
    \left\{
        r_T-\frac{\kappa T}{2}
    \right\}
    \le
    -A_{\gamma,t_0}.
\)
A simple sufficient condition is
\(
    \limsup_{T\to\infty}\frac{r_T}{T}
    <
    \frac{\kappa}{2}.
\)
In particular, any online learner with sublinear regret $r_T=o(T)$ has
$n_\gamma^{\rm ctx}(r)<\infty$. This is why standard OCO algorithms with
sublinear adversarial regret can be used as contextual skeptics: once their
learning cost is asymptotically smaller than the cumulative predictable edge,
Theorem~\ref{thm:power-linear-full-feature} enters the exponential-power
regime. While our theoretical framework so far is agnostic to the choice of the OCO algorithm, we summarize the update rules for common OCO algorithms (namely, OGD, FTRL, and ONS) in Table \ref{tab:oco-contextual-bets} of the Appendix~\ref{app:oco-certification-time}.
\end{remark}

\begin{remark}[Beyond uniform edge alternative \eqref{eq:uniform_edge}]
\label{remark:cumulative-edge-informal}
The uniform-edge alternative in Definition~\ref{def:linear-full-feature-alternative}
is a special case of a more general cumulative-edge alternative studied in Appendix~\ref{def:cumulative-full-feature-edge-alternative}. The cumulative alternative $\mathcal Q_{\rm cum}^{\phi}(D^{\kappa,t_0};\Phi)$ is defined by requiring that, for some
comparator $\theta^\star\in K$, the cumulative predictable edge
\(
    \sum_{t=1}^T b_t^Q(\theta^\star)
\)
is lower bounded by a deterministic envelope $D_T$. 
We then have \(
    \mathcal Q_{\kappa,t_0}^{\phi,\operatorname{lin}}(\Phi)
    \subseteq
    \mathcal Q_{\rm cum}^{\phi}(D^{\kappa,t_0};\Phi)
    \subseteq
    \mathcal Q^{\phi}
\)
(ref. Prop. \ref{prop:linear-alt-subset-full-alt}, Remark \ref{remark:connection_to_uniform_edge_cumulative}).
The corresponding
regret-to-power theorem (Thm. \ref{thm:full-feature-regret-to-power}) shows that 
power only requires cumulative feature-aligned signal to dominate the
Ville threshold and the learner's regret; it need not be uniformly positive at
every time.

\end{remark}

\begin{figure}[t]
    \centering
    \includegraphics[width=0.475\columnwidth]{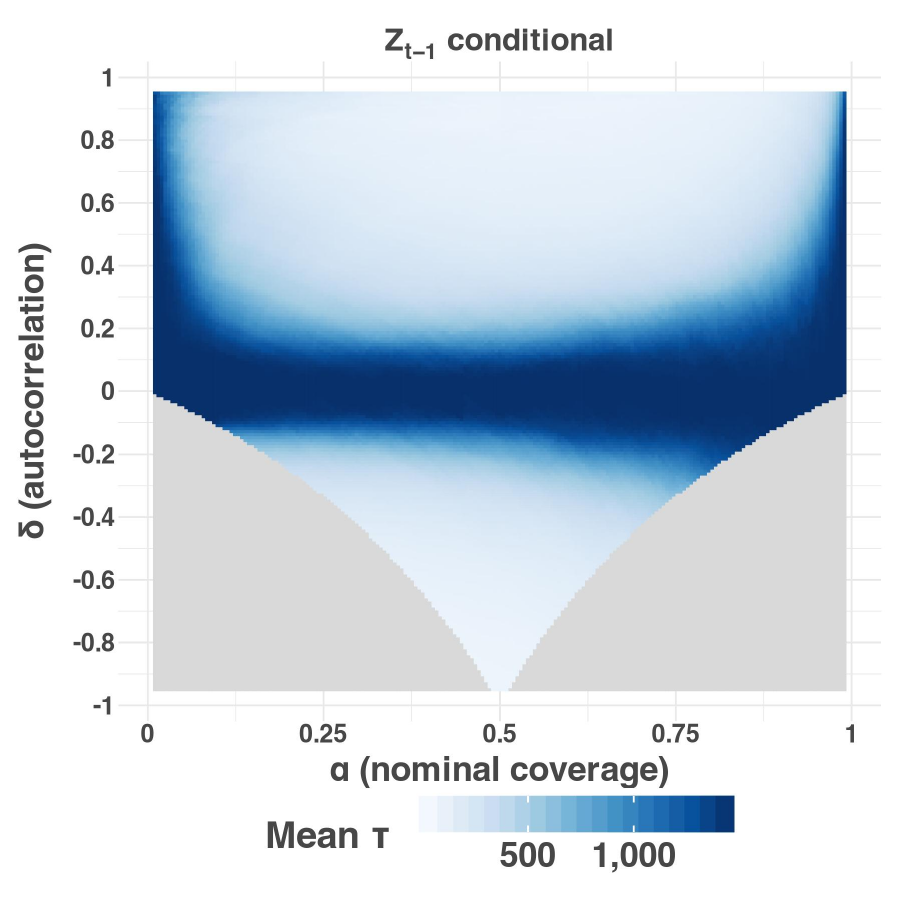}

    \caption{Sticky-coin experiment: mean rejection time $\bar\tau$ of
    the $Z_{t-1}$-aware skeptic over the $(\alpha,\delta)$ grid. Light cells
    reject early, dark blue cells are censored at $T$, grey pairs are
    infeasible.}
    \label{fig:sticky_zprev}
\end{figure}



\section{Experiments}
\label{sec:experiments}

Our experiments are organized around the validity--power separation developed
above. We ask three questions. First, can a coarser skeptic be valid but blind
when the hit process is non-i.i.d. and the violation is visible only through
richer information? Second, under a correctly specified oracle forecaster, do the
betting processes respect the Ville-prescribed Type-I level? Third, when the
oracle is replaced by a black-box quantile forecaster, do feature-aware skeptics
detect conditional miscalibration missed by marginal audits?
In the main-text experiments, for a
test level $\gamma_{\rm test}=0.05$, we record a rejection if
$\tau_{\gamma_{\rm test}}\le T$, equivalently if the running wealth crosses the
anytime-valid threshold $1/\gamma_{\rm test}=20$ by the end of the horizon.
Our skeptics use the FTRL configuration of
Appendix~\ref{app:simulation_protocol} and also test the Moirai-2 forecaster (Appendix \ref{app:additional_simulations}). Our baseline is two-sided GREM-90 e-backtest as described in Appendix~\ref{app:ebacktesting-comparison}.

\paragraph{Sanity check: Mean rejection times under conditional monitoring.}
To address the first question, we generate a binary Markov hit process
with nominal marginal hit probability $\alpha$ and dependence on the
preceding hit. This ``sticky-coin'' experiment removes the forecasting
layer to isolate the effect of monitoring information.
The intercept-only marginal skeptic typically remains censored at $T$
or rejects near it, whereas the $Z_{t-1}$-aware skeptic rejects quickly
across much of the feasible parameter grid, illustrating
Proposition~\ref{prop:coarser-tests-can-miss}.
Figure~\ref{fig:sticky_zprev} maps the mean rejection time of the
$Z_{t-1}$-aware skeptic over $\alpha$ and the one-step autocorrelation
$\delta$ ($N{=}50$ runs, $T{=}1{,}460$, $\gamma_{\rm test}{=}0.01$ for this
grid); Appendix~\ref{app:sticky_coin} gives the analogous checks for
the marginal and the e-backtest skeptics.

\begin{figure*}[t!]
    \centering
    \includegraphics[width=0.95\textwidth]{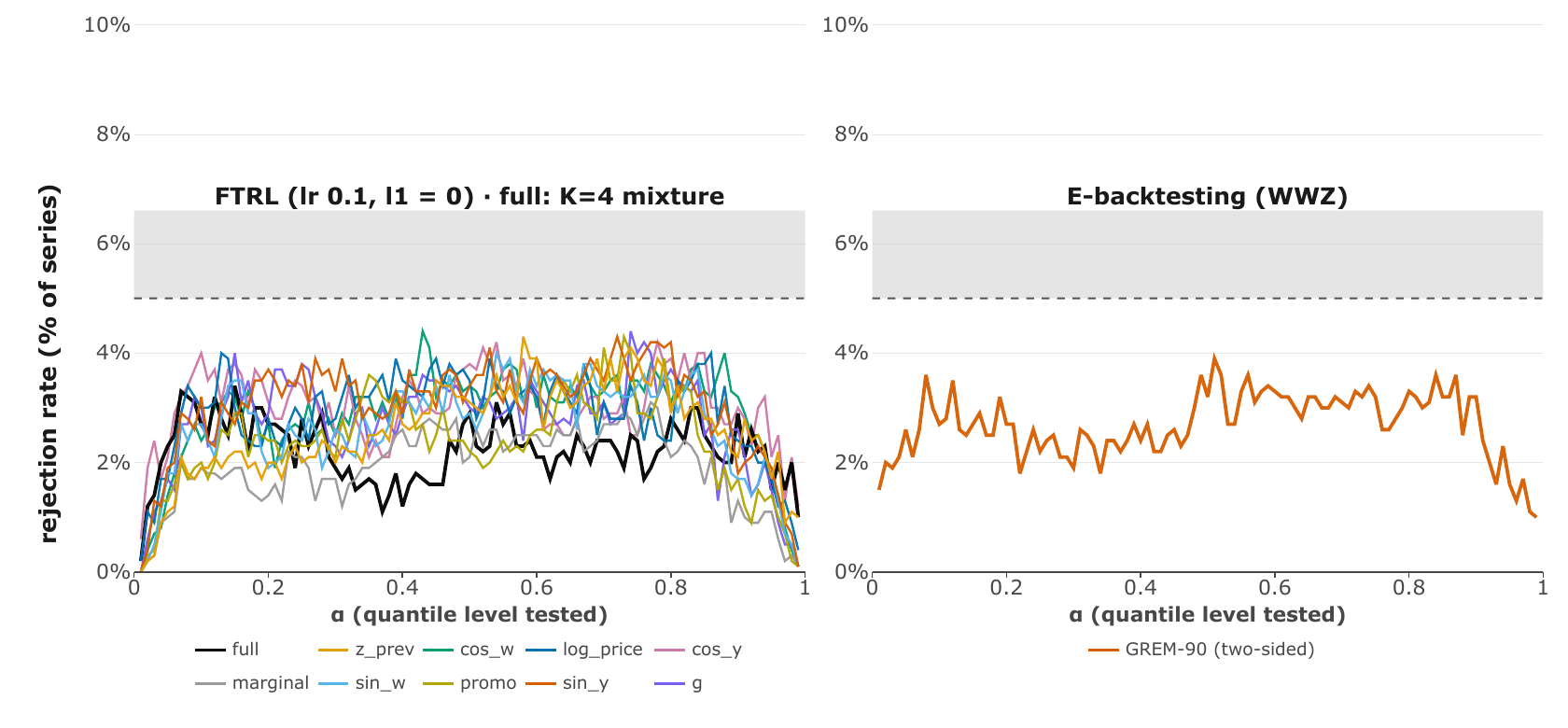}
    \vspace{-0.45cm}
    \caption{Type-I rejection rate versus quantile level $\alpha$ on
    $N{=}1{,}000$ calibrated-oracle series ($T{=}1{,}460$,
    $\gamma_{\rm test}=0.05$): FTRL skeptics (left,
    Appendix~\ref{app:simulation_protocol}) and the two-sided GREM-90
    e-backtest (right, Appendix~\ref{app:ebacktesting-comparison}). The grey
    band is a one-sided $99\%$ Monte Carlo tolerance band.}
    \label{fig:validity}
\end{figure*}

\paragraph{Validity under the null.}
\label{sec:validity}
We next instantiate the full forecast-and-bet protocol under a calibrated
synthetic oracle. The data are $N{=}1{,}000$ count series over $T{=}1{,}460$
days from the Negative-Binomial panel described in
Appendix~\ref{app:simulation_protocol}. At each quantile level
$\alpha\in\{0.01,0.02,\ldots,0.99\}$, the oracle issues the corresponding
Negative-Binomial forecast under the true parameters. Thus any rejection in this
experiment is a false alarm for the intended calibration benchmark.
For each pair of series and quantile level, we run the betting audit with the
same feature configurations used below and reject when
$\tau_{\gamma_{\rm test}}\le T$. Figure~\ref{fig:validity} shows that the
empirical rejection rates track the nominal $5\%$ Ville level across quantile
levels and skeptic configurations. This confirms the main validity claim in a
fully specified forecast-and-outcome setting before replacing the oracle by a
black-box forecaster.

\paragraph{Rejection rates on quantile forecasters on simulated and real sales datasets.}
\label{sec:realdata}
Finally, we replace the calibrated oracle with Chronos-2 forecasts
\citep{ansari2025chronos}. We evaluate two panels. 
The first reuses the first $500$ synthetic Negative-Binomial series from the validity experiment, but now the
forecasted quantiles are produced by Chronos-2 rather than by the true conditional law. 
The second is the Rossmann store-sales dataset \citep{cukierski2015rossmann} ($1{,}115$ stores), where the true conditional law is unknown. In both cases, the goal is no longer to measure Type-I error under a known null, but
to ask which monitoring views expose forecast miscalibration.
For each quantile level, we run marginal, single-feature, and full-feature skeptics, using only information available before $Y_t$ is observed. 
The marginal skeptic uses an intercept-only betting direction; a single-feature skeptic uses
one coordinate of the predictable feature dictionary, such as promotion status, Saturday, or log-price; and the full-feature skeptic uses the full dictionary.
{We also add the equal-weight wealth average of all specialist skeptics, denoted as the (mixture of) ``grand specialists'' (Appendix~\ref{app:simulation_protocol}); this mixture is an e-process for the intersection null $\cap_{w \in \phi} \mathcal{P}_0^w$, in the full-feature filtration $\mathbb{H}^\phi$. 
The grand-specialist mixture is conservative relative to the full-feature skeptic, as the intersection null is strictly larger than the full-feature null ($\mathcal{P}_0^\phi \subsetneq \cap_{w \in \phi} \mathcal{P}_0^w$). 
Yet, it can be more powerful in practice when the skeptic has many features.}

\begin{figure*}[t!]
    \centering
    \includegraphics[width=0.95\textwidth]{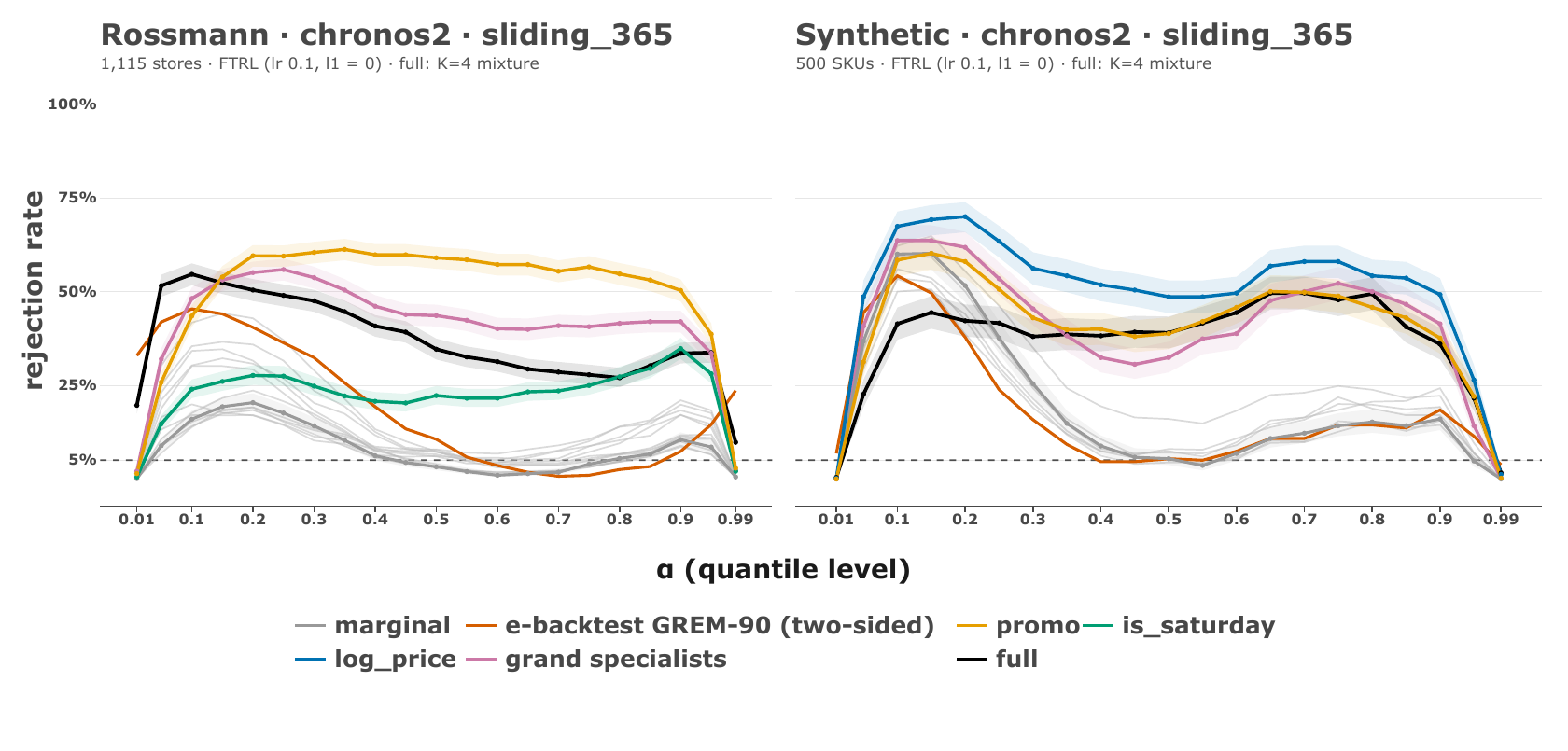}
    \vspace{-0.75cm}
    \caption{Per-feature rejection rate versus quantile level $\alpha$ on
    Chronos-2 forecasts: Rossmann (left) and the synthetic panel (right).
    Faded grey curves are the remaining single-feature skeptics; bands are
    $95\%$ Wilson intervals and the dashed line marks $5\%$.}
    \label{fig:chronos_panel}
\end{figure*}

Figure~\ref{fig:chronos_panel} reports per-feature rejection rates. 
The same experiments with a second forecaster, Moirai-2
\citep{liu2025moirai}, are reported in Appendix~\ref{app:additional_simulations}. 
{Generally, we find that skeptics aware of the relevant features for the audit reject far more often than the marginal and e-backtest skeptics.
The strongest signals are promotion status on both panels, Saturday on Rossmann, and log-price on the synthetic data, illustrating that the learned evidence is not only more powerful but also reflective of whether the skeptic's feature is useful for the audit. 
Both the full-feature and grand-specialist skeptics test the full-feature null hypothesis $\mathcal{P}_0^\phi$, and rejection suggests miscalibration w.r.t.~the full feature dictionary. 
We note that, because a skeptic with a single most indicative feature may detect miscalibration more quickly, using more features does not necessarily lead to higher rejection rates. 
We also note that, in settings with many potentially relevant features, the single-feature e-processes can be combined into a multiple testing procedure~\citep[e.g.,][]{wang2022ebh} at any non-feature-specific stopping time; the resulting feature \emph{selection} problem involves rejecting a subset of the null hypotheses $\{\mathcal{P}_0^w: w \in \phi\}$. 
We leave a comprehensive treatment of this problem as future work, noting the subtlety arising from having e-processes defined in multiple filtrations~\citep[e.g.,][]{choe2026combining,wang2025stopped_ebh,tavyrikov2026carefree}.}


\section{Conclusion}
\label{sec:discussion}

We developed an information-aware testing-by-betting framework for auditing
black-box conditional quantile forecasters. For any monitoring information
available to the auditor, predictable no-bankruptcy bets generate anytime-valid
e-processes for the corresponding calibration null, without assuming
independence, stationarity, parametric structure, or moment conditions on the
outcomes. The resulting null hierarchy separates validity from power: coarser
audits remain safe, but can be blind to violations visible only through richer
features. To gain feature-specific power, we introduced contextual betting, in
which an online learner adapts linear bets over the auditor's predictable feature
dictionary. Under persistent, or more generally cumulative, feature-aligned
predictable edge, pathwise regret bounds translate into finite-time detection
and stopping-time guarantees. Empirically, the method controls Type-I error under
a calibrated oracle and detects conditional miscalibration missed by marginal
audits on controlled non-i.i.d.\ streams and on Chronos-2 forecasts for synthetic
and Rossmann store-sales data. Comparisons with the two-sided GREM-90 e-backtest further demonstrate the value of contextual betting. 

Several directions remain open. Our contextual bets are linear in the supplied feature dictionary and might consequently fail to be sufficiently representative of miscalibration in high dimensions. Therefore, extending the framework to nonlinear betting classes
(e.g., kernelized skeptics) while preserving validity and obtaining
useful regret-to-power guarantees is an important next step. Our power guarantees
are necessarily feature-relative: if the auditor does not observe features that
expose the violation, the corresponding audit may remain valid but powerless.
The current framework also tests one quantile level at a time; aggregating
evidence across quantiles, or comparing two forecasters head-to-head, are natural
extensions of the same betting machinery. Finally, our experiments focus on
count-valued demand data and the Chronos-2 and Moirai-2
forecasters; broader benchmarks across continuous, multivariate, and non-retail
settings, as well as functionals beyond quantiles, are left for
future work.


\subsection*{AI use statement}
We used an LLM (GPT-5.6 Sol) as an assistive tool for:
(i) improving exposition via language editing; and
(ii) aiding understanding of theoretical concepts.
All LLM-generated suggestions were independently verified by the co-author team.

{\setstretch{1.0}
\bibliographystyle{plainnat}
\bibliography{src/references}
}

\clearpage
\appendix
\counterwithin{table}{section}
\renewcommand{\thetable}{\thesection.\arabic{table}}
\raggedbottom

\begin{center}
\Large\bfseries
Bet on Features: Anytime-Valid and Feature-Aware Auditing of Conditional Quantile Forecasters:\\
Supplementary Materials
\end{center}

\section{Related work}
\label{sec:related}

\paragraph{Forecast evaluation and conditional calibration.}
Classical forecast evaluation already identifies the core failure mode motivating
our work: forecasts can be calibrated on average while being predictably wrong
after conditioning on information available at forecast time.
For interval forecasts, \citet{christoffersen1998evaluating} showed that
unconditional coverage alone is inadequate when violations cluster dynamically.
For density and distributional forecasts, PIT diagnostics and calibration
taxonomies clarify how marginal, exceedance, and probabilistic calibration differ
\citep{diebold1998evaluating,gneiting2007probabilistic}. In risk management,
the analogous single-quantile problem is value-at-risk backtesting
\citep{christoffersen2004backtesting,engle2004caviar}. Our work studies the
sharp conditional quantile version of this program. Rather than evaluating an
entire predictive distribution, we audit a deployed black-box quantile forecast
$\widehat q_t$ through the hit
\(
    Z_t=\mathbf 1\{Y_t\le \widehat q_t\}-\alpha.
\)
This is narrower than full distributional calibration, but it is exactly the
property relevant to one-tail decisions such as inventory control and VaR-based
risk management. The main novelty is that we index the calibration null by the
auditor's monitoring information, thereby separating marginal validity from
feature-specific detectability.

\paragraph{Safe anytime-valid inference and sequential quantile methods.}
Classical backtests are typically fixed-horizon, whereas deployed forecasting
systems are monitored continuously. Safe anytime-valid inference addresses this
problem using nonnegative test supermartingales, e-values, and e-processes, whose
validity is preserved under optional stopping
\citep{ville1939etude,shafer2011test,shafer2021testing,vovk2021evalues,ramdas2022testing,ramdas2023game,grunwald2024safe,ramdas2025ebook}.
This viewpoint is natural in prequential forecast evaluation, where forecasts and
outcomes arrive sequentially \citep{1999_dawid_prequential,shafer2019game}.

Several recent works apply this methodology to quantiles, CDFs, or forecast
calibration, but with different statistical targets. \citet{howard2022sequential}
construct time-uniform confidence sequences for population quantiles from an
i.i.d. stream, with applications to A/B testing and best-arm identification. Their
object is quantile estimation from samples; our object is testing whether a
random, previously issued forecast $\widehat q_t$ satisfies
$P(Y_t\le \widehat q_t\mid\mathcal H_{t-1})=\alpha$. This extra layer is crucial:
even if one can estimate a population quantile sequentially, auditing a black-box
conditional quantile forecaster requires testing the conditional hit probability
of the issued forecasts. \citet{mineiro2023timeuniform} extend time- and
value-uniform CDF inference to nonstationary streams by targeting the running
average of conditional distributions. This is complementary to our work: they
estimate an entire changing CDF, whereas we test a forecast-specific conditional
calibration null and study feature-aware power. \citet{henzi2022sequentialcalibration}
develop sequentially valid tests for calibration of probabilistic forecasts,
closely related to PIT/probabilistic calibration. This is of separate interest,
but it requires distributional forecasts and is not directly recoverable from
quantile forecasts alone.

\paragraph{Elicitable functionals, sequential testing, and e-backtesting.}
Quantiles are elicitable and identifiable functionals: the pinball loss is
strictly consistent, and the identification function
$V_\alpha(y,q)=\mathbf 1\{y\le q\}-\alpha$ vanishes at an $\alpha$-quantile
\citep{koenker1978regression,gneiting2011quantiles,fissler2016higher}. A close predecessor of our work is the one by \citet{casgrain2024sequential},
who develop sequential tests for general elicitable and identifiable
functionals via supermartingales and predictable mixing, and use OCO regret to
obtain power guarantees. Our test can be viewed as a quantile-hit specialization
of this general paradigm, but with a different null and an extension of their (non-i.i.d.) asymptotic power to a non-asymptotic result. We audit random, covariate-dependent forecasts
$\widehat q_t$ and explicitly allow the skeptic's monitoring information
$\mathbb H$ to be a strict sub-filtration of the full pre-outcome information.
This leads to a hierarchy of nulls and to feature-specific power statements:
a coarser valid test can be powerless against violations visible only through
richer features. In contrast, the standard elicitable-functional setup
conditions on the full available filtration. Moreover, by specializing to the
bounded two-point hit process $Z_t\in\{-\alpha,1-\alpha\}$, our validity theory
requires no moment, stationarity, parametric, or tail assumptions on $Y_t$.

The e-backtesting framework of \citet{ebacktesting} is also closely related.
They introduce backtest e-statistics and construct e-processes for risk-measure
forecasts, especially VaR and ES, motivated particularly by model-free,
anytime-valid backtesting of ES in financial regulation.
Their VaR e-statistic satisfies
$e_t=\mathbf 1\{Y_t>\widehat q_t\}/(1-\alpha)
=1-Z_t/(1-\alpha)$,
so it equals our one-step wealth multiplier at
$\lambda=-1/(1-\alpha)$, the left endpoint of the no-bankruptcy
interval $\Lambda_\alpha$.
Investing a predictable fraction $u_t\in[0,1]$ gives
$1-u_t+u_te_t=1+\lambda_tZ_t$ with
$\lambda_t=-u_t/(1-\alpha)\le0$.
Their standard null conditions on the full market information filtration
and permits conservative forecasts, reflecting a focus on risk-measure
underestimation. Their Remark~5 also describes two-sided VaR testing by
averaging two one-sided e-processes;
Appendix~\ref{app:ebacktesting-comparison} presents our two-sided exact-hit
comparison in this notation and specifies the comparison baselines.
Both frameworks provide anytime validity under their respective conditional
nulls without i.i.d. assumptions.
Their GREE/GREL strategies choose scalar investment fractions through
empirical log-wealth maximization, with asymptotic optimality and power
guarantees under assumptions such as i.i.d. losses or i.i.d. e-statistics.
Although their general framework also permits predictable
covariate-dependent bets, we explicitly index calibration by the auditor's information and learn
contextual signed bets $\lambda_t=\langle\theta_t,\phi_t\rangle$ through
online convex optimization (OCO), converting
pathwise regret bounds into finite-time power guarantees for non-i.i.d.,
history-dependent streams with persistent feature-aligned predictable edge.

\paragraph{Conformal prediction, monitoring, and online learning.}
Our goal is also distinct from conformal prediction. Conformal prediction and
conformalized quantile regression construct prediction sets or intervals with
coverage guarantees \citep{vovk2005algorithmic,romano2019conformalized,angelopoulos2023gentle},
and adaptive conformal methods \citep{gibbs2021adaptive} primarily guarantee marginal long-run coverage under distribution shift, which does not resolve the conditional calibration gaps that the proposed feature-aware audit is specifically designed to detect. In other words, we instead audit a fixed deployed forecaster: we do not wrap, recalibrate, or
modify its predictions. The question is whether its realized errors are
conditionally exploitable by an auditor with a specified information set.

Finally, our contextual betting procedures use online convex optimization to
learn feature-aware bets. Classical OCO algorithms such as projected online
gradient descent, follow-the-regularized-leader, online Newton step, and
coin-betting provide regret guarantees for adapting decisions from sequential
data \citep{zinkevich2003online,hazan2007logarithmic,ShalevShwartz2012OCO,orabona2016coin}.
In our setting, regret has a statistical interpretation: it is the finite learning
cost paid before the skeptic can exploit feature-aligned miscalibration. Thus the
paper ties together classical conditional forecast evaluation, anytime-valid
testing, elicitable-functional martingales, and contextual online learning into
a feature-aware audit for black-box conditional quantile forecasters.

\section{Methodological details}

\subsection{Game protocol}
\label{apx:game_protocol}

\begin{definition}[Protocol for conditional quantile calibration game]
\label{def:filtration}
Fix $\alpha\in(0,1)$ and a monitoring filtration
$\mathbb H=(\mathcal H_t)_{t\ge0}$ satisfying
$\mathcal H_t\subseteq\mathcal F_t$.
As in Section~\ref{sec:preliminaries},
$\mathcal F_{t-1}$ denotes the richest pre-outcome information,
whereas $\mathcal H_{t-1}$ contains the information available
to Skeptic before betting. Skeptic starts with wealth $M_0=1$. For each round $t=1,2,\ldots$:
\begin{enumerate}
    \item The predictable context $C_t$, which is
    $\mathcal F_{t-1}$-measurable, is available to Forecaster.
    Skeptic observes only the context information included
    in its declared monitoring view $\mathcal H_{t-1}$.

    \item Forecaster issues an $\mathcal F_{t-1}$-measurable
    forecast $\widehat q_t\in\mathbb R$.
    Any forecast information available to Skeptic before
    betting is included in $\mathcal H_{t-1}$.

    \item Skeptic chooses an $\mathcal H_{t-1}$-measurable
    betting multiplier
    $\lambda_t\in\Lambda_\alpha
    :=[-1/(1-\alpha),\,1/\alpha]$.

    \item Reality reveals the $\mathcal F_t$-measurable
    outcome $Y_t$. The resulting centered hit
    \(
        Z_t:=\mathbf 1\{Y_t\le\widehat q_t\}-\alpha
    \)
    is recorded in $\mathcal H_t$.

    \item Skeptic updates its wealth as
    \(
        M_t=M_{t-1}(1+\lambda_t Z_t).
    \)
\end{enumerate}
\end{definition}

The process $(M_t)$ has the betting interpretation that Skeptic
starts with wealth one and repeatedly places
$\mathbb H$-predictable bets against the nominal hit probability
$\alpha$, using only its monitoring information.
Under every $P\in\mathcal P_0(\mathbb H)$, each legal one-step
bet is conditionally fair relative to $\mathcal H_{t-1}$, since
\(
\mathbb E_P[1+\lambda_t Z_t\mid\mathcal H_{t-1}]
=
1+\lambda_t\mathbb E_P[Z_t\mid\mathcal H_{t-1}]
=1.
\)
Consequently, $(M_t)$ is a nonnegative martingale with respect
to the monitoring filtration $\mathbb H$ under every
$P\in\mathcal P_0(\mathbb H)$.

\subsection{OCO algorithms}
\label{app:oco-certification-time}

In this subsection, we record three standard OCO algorithms, whose variants are used in our experimental analysis. We also introduce the regret rates for the corresponding algorithms from standard sources. Throughout, the learner chooses
$\theta_t\in K$ before observing $Z_t$, uses the betting fraction
\(
    \lambda_t=\langle\theta_t,\phi_t\rangle,
\)
and then observes the convex loss
\(
    g_t(\theta)
    :=
    -\log\left(1+\langle\theta,\phi_t\rangle Z_t\right).
\)
On $K=\{\theta:\|\theta\|_2\le 1/(2R)\}$, the denominator in
$\nabla g_t$ is at least $1/2$, so
\[
    \|\nabla g_t(\theta)\|_2
    =
    \left\|
        -\frac{Z_t\phi_t}
        {1+\langle\theta,\phi_t\rangle Z_t}
    \right\|_2
    \le 2R.
\]
Thus the losses are convex and uniformly Lipschitz on $K$. Moreover,
$\exp(-g_t(\theta))=1+\langle\theta,\phi_t\rangle Z_t$ is affine and
positive on $K$, so $g_t$ is exp-concave on $K$.

\paragraph{Online gradient descent.}
Projected online gradient descent updates in the direction of the negative
observed gradient and projects back to $K$:
\[
    \theta_{t+1}^{\rm OGD}
    =
    \Pi_K
    \left(
        \theta_t^{\rm OGD}
        -
        \eta_t\nabla g_t(\theta_t^{\rm OGD})
    \right),
\]
where $\Pi_K$ is Euclidean projection. For convex Lipschitz losses over a
bounded Euclidean domain, projected OGD achieves $O(\sqrt T)$ static regret
against the best fixed comparator in hindsight
\citep[Theorem~1]{zinkevich2003online};
see also \citet[Corollary~2.7]{ShalevShwartz2012OCO}.

\paragraph{Follow-the-regularized-leader.}
Euclidean FTRL chooses the next betting parameter by minimizing past losses
plus a stabilizing quadratic regularizer:
\[
    \theta_{t+1}^{\rm FTRL}
    =
    \arg\min_{\theta\in K}
    \left\{
        \sum_{i=1}^t g_i(\theta)
        +
        \frac{1}{2\eta}\|\theta\|_2^2
    \right\}.
\]
With $\eta=\Theta(T^{-1/2})$ when the horizon is known, or with a standard
doubling trick when it is not, Euclidean FTRL has $O(\sqrt T)$ regret for
convex Lipschitz losses. This follows from the strongly convex regularizer
analysis of Follow-the-Regularized-Leader
\citep[Theorem~2.11 and Corollary~2.12]{ShalevShwartz2012OCO}.

\paragraph{Online Newton Step.}
Because the betting losses are exp-concave, one can exploit curvature using
Online Newton Step.
Initialize $\theta_1^{\rm ONS}=0$ and $A_0=\varepsilon I_d$,
where $\varepsilon>0$ is fixed before observing the data.
Since
$A_t=\varepsilon I_d+\sum_{i=1}^t\nabla_i\nabla_i^\top
\succeq\varepsilon I_d$,
the inverse $A_t^{-1}$ exists at every round.
With
\[
    \nabla_t:=\nabla g_t(\theta_t^{\rm ONS}),
    \qquad
    A_t=A_{t-1}+\nabla_t\nabla_t^\top,
\]
the update is
\[
    \theta_{t+1}^{\rm ONS}
    =
    \Pi_K^{A_t}
    \left(
        \theta_t^{\rm ONS}
        -
        \eta A_t^{-1}\nabla_t
    \right),
\]
where
\(
\Pi_K^{A}(y)
    :=
    \arg\min_{\theta\in K}
    (\theta-y)^\top A(\theta-y).
\)
For exp-concave losses with bounded gradients over a bounded domain, ONS
with the standard parameter choices
achieves logarithmic regret, specifically $O(d\log T)$ in dimension $d$
\citep[Theorem~2]{hazan2007logarithmic}.


\begin{table}[t]
\centering
\small
\setlength{\tabcolsep}{4pt}
\renewcommand{\arraystretch}{1.15}
\begin{tabularx}{\linewidth}{
  @{}
  p{0.27\linewidth}
  >{\raggedright\arraybackslash}p{0.23\linewidth}
  >{\raggedright\arraybackslash}X
  @{}
}
\toprule
Algorithm
&
Regret guarantee
&
\multicolumn{1}{c@{}}{Contextual betting update}
\\
\midrule
Projected OGD
\citep[Theorem~1]{zinkevich2003online};
\citep[Corollary~2.7]{ShalevShwartz2012OCO}
&
$R_T^{\rm ctx}(\theta)=O(\sqrt T)$ for convex Lipschitz losses
&
\[
\theta_{t+1}
=
\Pi_K\!\left(
    \theta_t-\eta_t\nabla g_t(\theta_t)
\right).
\]
\\
\midrule
Euclidean FTRL
\citep[Theorem~2.11 and Corollary~2.12]{ShalevShwartz2012OCO}
&
$R_T^{\rm ctx}(\theta)=O(\sqrt T)$ for convex Lipschitz losses
&
\[
\theta_{t+1}
=
\arg\min_{\theta\in K}
\left\{
    \sum_{i=1}^t g_i(\theta)
    +
    \frac{1}{2\eta}\|\theta\|_2^2
\right\}.
\]
\\
\midrule
Online Newton Step
\citep[Theorem~2]{hazan2007logarithmic}
&
$R_T^{\rm ctx}(\theta)=O(d\log T)$ for exp-concave losses
&
\[
\begin{aligned}
A_t
&=
A_{t-1}+\nabla_t\nabla_t^\top,
\\
\theta_{t+1}
&=
\Pi_K^{A_t}
\left(
    \theta_t-\eta A_t^{-1}\nabla_t
\right).
\end{aligned}
\]
\\
\bottomrule
\end{tabularx}
\caption{Standard OCO algorithm updates for learning contextual betting directions.
Here $g_t(\theta)=-\log(1+\langle\theta,\phi_t\rangle Z_t)$,
$\nabla_t=\nabla g_t(\theta_t)$, and the next-round bet is given by
$\lambda_{t+1}=\langle\theta_{t+1},\phi_{t+1}\rangle$.
ONS is initialized with $\theta_1=0$ and
$A_0=\varepsilon I_d$ for a fixed $\varepsilon>0$.
}
\label{tab:oco-contextual-bets}
\end{table}

The implication of the regret rates presented in Table~\ref{tab:oco-contextual-bets} for our stopping time result in Theorem~\ref{thm:power-linear-full-feature} is immediate:
all three algorithms  have sublinear
regret in the adversarial OCO sense, and hence satisfy
\(
    \limsup_{T\to\infty}
    \frac{r_T}{T}
    =
    0
    <
    \frac{\kappa}{2}
\)
for every $0<\kappa\le\kappa_{\max}(\alpha)$. Therefore the certification time
$n_\gamma^{\rm ctx}(r)$ is finite for each of these learners. The difference
is quantitative: OGD and Euclidean FTRL give $\sqrt T$ learning cost, whereas
ONS gives logarithmic learning cost for the exp-concave betting losses.

\subsection{E-backtesting in the quantile-hit formulation}
\label{app:ebacktesting-comparison}

{In this section, we summarize the e-backtesting~\citep{ebacktesting} procedure for quantiles (VaR), to provide a fair comparison with our proposed method. Denote $B_t=\mathbf1\{Y_t\le\widehat q_t\}$ and $Z_t=B_t-\alpha$, for fixed
$\alpha\in(0,1)$. The standard null for one-sided, $\alpha$-quantile e-backtesting is
$\widehat q_t\ge Q_\alpha(Y_t\mid\mathcal F_{t-1})$, or equivalently, $P(B_t=1\mid\mathcal F_{t-1})\ge\alpha$ at every round. 
Qualitatively, this means that we allow conservative forecasts but (ideally) flag underestimation of the $\alpha$-quantile.
The e-statistic and the resulting wealth multiplier are given by
\[
e_t^{\mathrm{lo}}=\frac{1-B_t}{1-\alpha},\qquad
1-u_t+u_te_t^{\mathrm{lo}}
=1-\frac{u_t}{1-\alpha}Z_t,
\]
where, the predictable betting fraction $u_t\in[0,1]$ correspond to the signed bet $\lambda_t=-u_t/(1-\alpha)\in[-1/(1-\alpha),0]$ in Definition~\ref{def:skeptic-wealth-process}. 
As the resulting test only monitors one-sided miscalibration, we look at a two-sided exact-hit version of e-backtesting (Remark~5 in their paper) and add
$e_t^{\mathrm{hi}}=B_t/\alpha$, whose multiplier is
$1-v_t+v_te_t^{\mathrm{hi}}=1+(v_t/\alpha)Z_t$. For predictable betting fractions on either side,  $u_t,v_t\in[0,c]$ with $c=1/2$, initialize both accounts at one and use
\[
M_t^s=\prod_{i=1}^t(1-a_i^s+a_i^se_i^s),\qquad
a_i^{\mathrm{lo}}=u_i,\quad a_i^{\mathrm{hi}}=v_i,\qquad
M_t^{\mathrm{2s}}=\frac{M_t^{\mathrm{lo}}+M_t^{\mathrm{hi}}}{2}.
\]
Under the two-sided calibration null, $P(B_t=1\mid\mathcal F_{t-1})=\alpha$, both e-statistics have
conditional mean one. Predictability and positive multipliers make both
accounts and their mean test martingales, without independence or
stationarity. We reject when $M_t^{\mathrm{2s}}\ge1/\gamma_{\rm test}$
($20$ at level $0.05$). The mean satisfies
$\log M_t^{\mathrm{2s}}\ge\log M_t^s-\log2$ for either side. 

The two betting strategies proposed by~\citet{ebacktesting}, GREE and GREL, both maximize empirical log-wealth. GREE uses past hits at
their issued forecasts; GREL evaluates past outcomes at the current
forecast. For index sets
$I_t^{\mathrm E}=\{1,\ldots,t-1\}$ and
$I_t^{\mathrm L}=\{\max(1,t-90),\ldots,t-1\}$ (assuming they are nonempty), we compute
\[
\widehat p_t^{\mathrm E}
=\frac{\sum_{s\in I_t^{\mathrm E}}\mathbf1\{Y_s\le\widehat q_s\}}
       {|I_t^{\mathrm E}|},\qquad
\widehat p_t^{\mathrm L}
=\frac{\sum_{s\in I_t^{\mathrm L}}\mathbf1\{Y_s\le\widehat q_t\}}
       {|I_t^{\mathrm L}|}.
\]

For either estimate $\widehat p_t$, each investment fraction is chosen to maximize its empirical average log-wealth increment.
The lower account has multiplier $1-u$ on a hit and
$1+\alpha u/(1-\alpha)$ on a miss, giving the objective
$g_{\rm lo}(u)=\widehat p_t\log(1-u)
+(1-\widehat p_t)\log(1+\alpha u/(1-\alpha))$.
The upper account rewards hits instead, with objective
$g_{\rm hi}(v)=\widehat p_t\log(1+v(1-\alpha)/\alpha)
+(1-\widehat p_t)\log(1-v)$.
Both objectives are concave; setting their derivatives to zero and
enforcing the investment range $[0,c]$ yields
$u_t=\operatorname{clip}_{[0,c]}((\alpha-\widehat p_t)/\alpha)$ and
$v_t=\operatorname{clip}_{[0,c]}((\widehat p_t-\alpha)/(1-\alpha))$.
Thus the lower account bets when estimated coverage is below $\alpha$,
the upper account bets when it is above $\alpha$, and $c$ limits the
fraction invested.

In the experiments both frequencies are smoothed,
$\widehat p_t=(k_1+\tfrac12)/(k+1)$ for a history of $k$ rounds with $k_1$
indicators equal to one, and both fractions are zero for the first $20$
rounds.
GREL-90 uses the window above; two-sided GREM-90 averages the two-sided
GREE and GREL-90 wealth processes, equivalently giving equal initial
weight to all four directional accounts. The current forecast and all
inputs used to choose fractions are included in the monitoring
information before observing $Y_t$.
Every e-backtesting result in this paper uses this two-sided GREM-90.

The contextual auditor proposed in this paper uses OCO to learn feature weights $\theta_t$ from past observations, then uses the current features $\phi_t$ to choose the signed bet $\lambda_t=\langle\theta_t,\phi_t\rangle$ before observing $Y_t$, adapting both its direction and magnitude to the forecast context. The e-backtesting baseline above uses no explicit covariates, although GREL uses the current forecast. The general e-backtesting framework also allows predictable
covariate-dependent fractions. Our distinction lies in explicitly indexing
calibration by the auditor's information and using feature-aware online
learning to obtain finite-time power under persistent feature-aligned
miscalibration.}

\section{Theoretical details}
\label{app:proofs}

\subsection{Missing definitions}

\begin{definition}[Conditional e-variables]
\label{def:conditional-evariables}
Fix a monitoring filtration $\mathbb H=(\mathcal H_t)_{t\ge0}$.
A nonnegative $\mathcal H_t$-measurable random variable $E_t$
is a conditional e-variable for $\mathcal P_0(\mathbb H)$
at time $t$ if, for every $P\in\mathcal P_0(\mathbb H)$,
\(
    \mathbb E_P[E_t\mid\mathcal H_{t-1}]
    \le 1
    \, P\text{-a.s.}
\)
If each $E_t$ is a conditional e-variable in this sense, then
the product process
\(
    M_0:=1,
    \,
    M_t:=\prod_{i=1}^t E_i
\)
is a nonnegative, $\mathbb H$-adapted test supermartingale
for $\mathcal P_0(\mathbb H)$, since
\(
    \mathbb E_P[M_t\mid\mathcal H_{t-1}]
    =
    M_{t-1}\mathbb E_P[E_t\mid\mathcal H_{t-1}]
    \le M_{t-1} \, P\text{-a.s.}
\)
for every $P\in\mathcal P_0(\mathbb H)$ and every $t\ge1$.
\end{definition}

\subsection{Supporting lemmas}

\begin{lemma}[Conditional Hoeffding bound with support-length control]
\label{lem:hoeffding-length}
Let $(m_t)_{t\ge1}$ be a martingale-difference sequence with respect to a filtration $(\mathcal F_t)_{t\ge0}$; that is, each $m_t$ is $\mathcal F_t$-measurable, integrable, and $\mathbb E[m_t\mid\mathcal F_{t-1}]=0$ a.s. Suppose there exist deterministic constants $c_t\ge0$ such that, for each $t$, there are $\mathcal F_{t-1}$-measurable random variables $a_t$ and $b_t$ satisfying $a_t\le m_t\le b_t$ a.s. and $b_t-a_t\le c_t$ a.s. Let $V_T:=\sum_{t=1}^T c_t^2$. Then, for every $T\ge1$ and $x>0$,
\[
    \mathbb P\!\left(\sum_{t=1}^T m_t\le -x\right)
    \le
    \exp\!\left(-\frac{2x^2}{V_T}\right),
    \qquad
    \mathbb P\!\left(\left|\sum_{t=1}^T m_t\right|\ge x\right)
    \le
    2\exp\!\left(-\frac{2x^2}{V_T}\right),
\]
with the convention that the right-hand side is $0$ when $V_T=0$.
\end{lemma}

\begin{proof}[Proof of Lemma \ref{lem:hoeffding-length}]
If $V_T=0$, then $c_t=0$ for every $t\le T$. Hence $b_t-a_t=0$ a.s. for every $t\le T$, so $m_t=a_t=b_t$ a.s.; since $\mathbb E[m_t\mid\mathcal F_{t-1}]=0$, it follows that $m_t=0$ a.s. for every $t\le T$. The claimed inequalities are then trivial, so assume $V_T>0$.

Fix $t$ and $\eta\in\mathbb R$. Conditional on $\mathcal F_{t-1}$, the random variable $m_t$ has conditional mean zero and is supported in $[a_t,b_t]$, whose length is at most $c_t$. Conditional Hoeffding's lemma gives $\mathbb E[e^{\eta m_t}\mid\mathcal F_{t-1}]\le\exp(\eta^2c_t^2/8)$ a.s. For $\eta\in\mathbb R$, define
\[
    L_t(\eta)
    :=
    \exp\!\left(
        \eta\sum_{i=1}^t m_i
        -
        \frac{\eta^2}{8}\sum_{i=1}^t c_i^2
    \right),
    \qquad
    L_0(\eta):=1.
\]
Then $(L_t(\eta))_{t=0}^T$ is a nonnegative supermartingale, because
\[
    \mathbb E[L_t(\eta)\mid\mathcal F_{t-1}]
    =
    L_{t-1}(\eta)\exp\!\left(-\frac{\eta^2c_t^2}{8}\right)
    \mathbb E[e^{\eta m_t}\mid\mathcal F_{t-1}]
    \le
    L_{t-1}(\eta).
\]
Consequently $\mathbb E[L_T(\eta)]\le1$, equivalently $\mathbb E\exp(\eta\sum_{t=1}^T m_t)\le\exp(\eta^2V_T/8)$.

For the lower tail, take $\eta>0$ and apply Markov's inequality to $\exp(-\eta\sum_{t=1}^T m_t)$:
\[
    \mathbb P\!\left(\sum_{t=1}^T m_t\le -x\right)
    \le
    e^{-\eta x}\,
    \mathbb E\exp\!\left(-\eta\sum_{t=1}^T m_t\right)
    \le
    \exp\!\left(-\eta x+\frac{\eta^2V_T}{8}\right).
\]
Optimizing over $\eta>0$ gives $\eta^\star=4x/V_T$, and therefore $\mathbb P(\sum_{t=1}^T m_t\le -x)\le\exp(-2x^2/V_T)$.

The same argument applied to $-m_t$ gives $\mathbb P(\sum_{t=1}^T m_t\ge x)\le\exp(-2x^2/V_T)$. Combining the upper and lower tail bounds by the union bound yields $\mathbb P(|\sum_{t=1}^T m_t|\ge x)\le2\exp(-2x^2/V_T)$.
\end{proof}

\subsection{Missing results}
\label{apx:missing_results}

\begin{proposition}[Bounded martingale difference sequence property]
\label{prop:bounded-monitored-mds} If $P\in\mathcal P_0(\mathbb H)$, then $(Z_t)_{t\ge1}$ is a bounded martingale
difference sequence (MDS) with respect to $\mathbb H$: for every $t\ge1$,
\(
    Z_t \text{ is } \mathcal H_t\text{-measurable},
\)
\(
    Z_t\in\{-\alpha,1-\alpha\} \) \( P\text{-a.s.},
\)
and
\(
    \mathbb E_P[Z_t\mid\mathcal H_{t-1}]
    =
    0\) \( 
    P\text{-a.s.}
\)
\end{proposition}

\begin{proof}[Proof of Proposition \ref{prop:bounded-monitored-mds}]
Fix $P\in\mathcal P_0(\mathbb H)$.
By the definition of the monitoring information structure,
the centered hit
\(
    Z_t=\mathbf 1\{Y_t\le\widehat q_t\}-\alpha
\)
is recorded in $\mathcal H_t$ after the outcome is observed.
Thus $Z_t$ is $\mathcal H_t$-measurable for every $t\ge1$.
Since the hit indicator is binary,
$Z_t\in\{-\alpha,1-\alpha\}$, so $|Z_t|\le1$ and
$Z_t$ is integrable.
Moreover, the definition of $\mathcal P_0(\mathbb H)$ gives
\(
    \mathbb E_P[Z_t\mid\mathcal H_{t-1}]
    =
    P(Y_t\le\widehat q_t\mid\mathcal H_{t-1})-\alpha
    =0
    \, P\text{-a.s.}
\)
for every $t\ge1$.
Hence $(Z_t)_{t\ge1}$ is a bounded martingale difference
sequence with respect to $\mathbb H$.
\end{proof}

\begin{proposition}[Oracle Kelly betting strategy]
\label{prop:oracle}
For each $t$, the extended-real maximizer of $\psi_t$ over $\Lambda_\alpha$ is
\(
    \lambda_t^\star
    =
    \frac{p_t-\alpha}{\alpha(1-\alpha)}.
    \label{eq:oracle-lambda}
\)
If $p_t\in(0,1)$, this maximizer lies in $\operatorname{ri}(\Lambda_\alpha)=(-1/(1-\alpha),\,1/\alpha)$ and both endpoints have value $-\infty$. If $p_t=0$ or $p_t=1$, the same formula gives the corresponding no-bankruptcy endpoint. In all cases, $\psi_t(\lambda_t^\star)=\KL(\Ber(p_t)\,\|\,\Ber(\alpha))$.
\end{proposition}

\begin{proof}[Proof of Proposition \ref{prop:oracle}]
For a fixed stake $\lambda$, define the conditional expected log-growth
\(
    \psi_t(\lambda)
    :=
    \mathbb E[\log(1+\lambda Z_t)\mid\mathcal F_{t-1}]
    =
    p_t\log\{1+\lambda(1-\alpha)\}
    +(1-p_t)\log(1-\alpha\lambda),
\)
with conventions $0\log0:=0$ and $a\log0:=-\infty$ for $a>0$.
Fix $t$ and write $p:=p_t$. On $\operatorname{ri}(\Lambda_\alpha)$, $\psi_t$ is differentiable with derivative
\(
    \psi_t'(\lambda)
    =
    \frac{p(1-\alpha)}{1+\lambda(1-\alpha)}
    -
    \frac{(1-p)\alpha}{1-\alpha\lambda}.
\)
If $p\in(0,1)$, then $\psi_t$ is strictly concave on $\operatorname{ri}(\Lambda_\alpha)$, and solving $\psi_t'(\lambda)=0$ gives $\lambda=(p-\alpha)/\{\alpha(1-\alpha)\}$, which lies in $\operatorname{ri}(\Lambda_\alpha)$. At either endpoint one possible wealth multiplier is zero; since both outcomes have positive probability, the expected log-growth there is $-\infty$. If $p=0$, then $\psi_t(\lambda)=\log(1-\alpha\lambda)$ is decreasing and is maximized at $-1/(1-\alpha)$. If $p=1$, then $\psi_t(\lambda)=\log\{1+\lambda(1-\alpha)\}$ is increasing and is maximized at $1/\alpha$. Both endpoint cases agree with \eqref{eq:oracle-lambda}. Substituting \eqref{eq:oracle-lambda} gives $1+\lambda_t^\star(1-\alpha)=p/\alpha$ and $1-\alpha\lambda_t^\star=(1-p)/(1-\alpha)$, hence $\psi_t(\lambda_t^\star)=p\log(p/\alpha)+(1-p)\log\{(1-p)/(1-\alpha)\}=\KL(\Ber(p)\,\|\,\Ber(\alpha))$.
\end{proof}

\begin{lemma}[Quadratic lower bound on full-feature log-growth]
\label{lem:quadratic-contextual-edge}
For every law $Q$, every $\theta\in K$, and every $t\ge1$,
\[
\mathbb E_Q\!\left[
\log\!\left(1+\langle\theta,\phi_t\rangle Z_t\right)
\,\middle|\,
\mathcal H_{t-1}^\phi
\right]
\ge b_t^Q(\theta)
\qquad Q\text{-a.s.}
\]
Moreover, if $b_t^Q(\theta)>0$ on an event of positive probability, then $m_t^\phi\ne0$ on an event of positive probability. Hence any law admitting a uniformly positive full-feature quadratic edge violates $\mathcal P_0^\phi$.
\end{lemma}

\begin{proof}
Set $\lambda_t=\langle\theta,\phi_t\rangle$. Since $|\lambda_t Z_t|\le1/2$, the elementary inequality $\log(1+x)\ge x-x^2$ for $x\in[-1/2,1/2]$ gives $\log(1+\lambda_t Z_t)\ge \lambda_t Z_t-\lambda_t^2 Z_t^2$. Taking conditional expectation given $\mathcal H_{t-1}^\phi$, and using that $\lambda_t$ is $\mathcal H_{t-1}^\phi$-measurable, yields
\[
\begin{aligned}
\mathbb E_Q\!\left[
\log(1+\lambda_t Z_t)
\,\middle|\,
\mathcal H_{t-1}^\phi
\right]
&\ge
\lambda_t\mathbb E_Q[Z_t\mid\mathcal H_{t-1}^\phi]
-
\lambda_t^2\mathbb E_Q[Z_t^2\mid\mathcal H_{t-1}^\phi] \\
&\ge
\lambda_t m_t^\phi-\lambda_t^2,
\end{aligned}
\]
because $Z_t^2\le1$. This is exactly $b_t^Q(\theta)$.
If $m_t^\phi=0$, then $b_t^Q(\theta)=-\langle\theta,\phi_t\rangle^2\le0$. Therefore $b_t^Q(\theta)>0$ implies $m_t^\phi\ne0$. A uniformly positive edge therefore implies failure of the full-feature calibration null.
\end{proof}

\begin{remark}[Feasible range of the quadratic edge]
For every $\theta\in K$ and $t\ge1$, write
$\lambda=\langle\theta,\phi_t\rangle$ and $m=m_t^\phi$.
Since $m\in[-\alpha,1-\alpha]$,
\(
    b_t^Q(\theta)
    =\frac{m^2}{4}-\left(\lambda-\frac m2\right)^2
    \le\frac{m^2}{4}
    \le\kappa_{\max}(\alpha)<\frac14.
\)
Thus the uniform-edge class is empty above this bound.
The bound is attainable across admissible dictionaries and laws, as explained in the following example. For a fixed dictionary, attainable edges may be smaller; for example, a dictionary identically zero at a post-onset round precludes any positive
uniform edge.
\end{remark}

\begin{example}[Nondegenerate alternatives]
Positive uniform edge can arise from genuinely random hits.
We illustrate this with a persistent coverage bias that a constant
betting direction can exploit.
Fix $\alpha\in(0,1)$ and $R>0$. Take $d=1$, $\phi_t=R$, and let
$Q$ generate independent hits
$B_t\sim\operatorname{Bernoulli}(\alpha+\delta)$, where
$-\alpha<\delta<1-\alpha$ and $\delta\ne0$.
Assume that this law $Q$ belongs to the
ambient model class $\mathcal P$.
Both hit outcomes have positive probability. The parameter $\delta$
measures the coverage bias: hits occur more often than the nominal
rate when $\delta>0$, and less often when $\delta<0$.
With the past-hit filtration
$\mathcal H_t^\phi=\sigma(B_1,\ldots,B_t)$, independence gives
$m_t^\phi=\mathbb E_Q[B_t-\alpha\mid\mathcal H_{t-1}^\phi]=\delta$
at every round.

Choose the fixed comparator $\theta^\star=\delta/(2R)$, which produces
the constant bet $\lambda_t^\star=\delta/2$ in the direction of this bias.
Since $|\delta|<1$, we have $|\theta^\star|<1/(2R)$, so
$\theta^\star\in K$. It provides a fixed benchmark for the online learner.
Its quadratic edge is
\(
b_t^Q(\theta^\star)
=\lambda_t^\star m_t^\phi-(\lambda_t^\star)^2
=\frac{\delta^2}{2}-\frac{\delta^2}{4}
=\frac{\delta^2}{4}>0.
\)
The same comparator therefore supplies positive edge at every round,
giving $Q\in\mathcal Q_{\kappa,1}^{\phi,\operatorname{lin}}(\Phi)$
for every $0<\kappa\le\delta^2/4$.

If $\mathcal P$ contains the family of i.i.d.\
Bernoulli laws constructed here, this construction covers every
threshold $0<\kappa<\kappa_{\max}(\alpha)$ with this dictionary.
Indeed, choose $\delta=2\sqrt{\kappa}$ when $\alpha\le1/2$,
and $\delta=-2\sqrt{\kappa}$ otherwise.
The strict bound on $\kappa$ makes this choice admissible and gives
edge exactly $\kappa$, while keeping the hits nondegenerate.
The construction also fits the forecast-and-outcome protocol:
set $\widehat q_t=0$ and $Y_t=1-2B_t$.
Then $Y_t=-1$ when $B_t=1$ and $Y_t=1$ when $B_t=0$,
so $\mathbf1\{Y_t\le\widehat q_t\}=B_t$.
\end{example}

\begin{proposition}[Linear full-feature alternatives are full-feature alternatives]
\label{prop:linear-alt-subset-full-alt}
For every 
$0<\kappa\le\kappa_{\max}(\alpha)$
and $t_0\in\mathbb N_{\ge1}$,
we have that linear predictable-edge alternative is a subclass of $\mathcal Q^{\phi}$, i.e.,
\(
    \mathcal Q_{\kappa,t_0}^{\phi,\operatorname{lin}}(\Phi)
    \subseteq
    \mathcal P\setminus \mathcal P_0^\phi.
\)
\end{proposition}

\begin{proof}[Proof of Proposition \ref{prop:linear-alt-subset-full-alt}]
Let $Q\in\mathcal Q_{\kappa,t_0}^{\phi,\operatorname{lin}}(\Phi)$. Then there exists $\theta^\star\in K$ such that $b_t^Q(\theta^\star)\ge\kappa>0$ $Q$-a.s. for every $t\ge t_0$. By Lemma~\ref{lem:quadratic-contextual-edge}, $b_t^Q(\theta^\star)>0$ implies $m_t^\phi\ne0$. Therefore $\mathbb E_Q[Z_t\mid\mathcal H_{t-1}^\phi]=m_t^\phi\ne0$ with positive probability for every $t\ge t_0$. Hence $Q\notin\mathcal P_0^\phi$.
\end{proof}

\begin{definition}[Cumulative full-feature edge alternative]
\label{def:cumulative-full-feature-edge-alternative}
Fix a predictable feature sequence
$\Phi=(\phi_t)_{t\ge1}$ and a deterministic real sequence
$D=(D_T)_{T\ge1}$. For $\theta\in K$, define the predictable cumulative
full-feature edge
\(
    B_T^Q(\theta)
    :=
    \sum_{t=1}^T b_t^Q(\theta),
\)
where
\(
    b_t^Q(\theta)
    :=
    \langle\theta,\phi_t\rangle m_t^\phi
    -
    \langle\theta,\phi_t\rangle^2,
\)
and
\(
    m_t^\phi
    :=
    \mathbb E_Q[Z_t\mid \mathcal H_{t-1}^\phi].
\)
The cumulative full-feature edge alternative with lower envelope $D$ is
\[
    \mathcal Q_{\rm cum}^{\phi}(D;\Phi)
    :=
    \left\{
        Q\in\mathcal P:\;
        \exists \theta^\star\in K
        \text{ such that }
        B_T^Q(\theta^\star)
        \ge
        D_T
        \quad
        Q\text{-a.s. for every }T\ge1
    \right\}.
\]
\end{definition}

\begin{theorem}[Full-feature contextual regret-to-power under cumulative edge]
\label{thm:full-feature-regret-to-power}
Assume that the online learner outputs
$\mathcal H_{t-1}^\phi$-measurable vectors $\theta_t\in K$ and satisfies the
pathwise regret bound $R_T^{\rm ctx}(\theta)\le r_T$ for every $\theta\in K$ and every $T\ge1$, where $(r_T)_{T\ge1}$ is deterministic. Define
\(
    \tau_\gamma^{\rm ctx}
    :=
    \inf\{T\in\mathbb N_{\ge1}:M_T^{\rm ctx}\ge1/\gamma\}.
\)
Fix $\gamma\in(0,1)$ and a deterministic lower envelope
$D=(D_T)_{T\ge1}$. Suppose
\(
    Q\in \mathcal Q_{\rm cum}^{\phi}(D;\Phi).
\)
Define
\(
    s_T(D)
    :=
    D_T-r_T-\log(1/\gamma).
\)
Then, for every integer $T\ge1$ such that $s_T(D)>0$, we have
\[
    Q(\tau_\gamma^{\rm ctx}>T)
    \le
    \exp\!\left(
        -\frac{8s_T(D)^2}{T}
    \right).
\]
\end{theorem}

\begin{proof}[Proof of Theorem \ref{thm:full-feature-regret-to-power}]
Choose the fixed comparator $\theta^\star$ guaranteed by
$Q\in\mathcal Q_{\rm cum}^{\phi}(D;\Phi)$.
Set $\lambda_t^\star:=\langle\theta^\star,\phi_t\rangle$. By regret against $\theta^\star$,
\[
    \log M_T^{\rm ctx}
    =
    -\sum_{t=1}^T g_t(\theta_t)
    \ge
    -\sum_{t=1}^T g_t(\theta^\star)-r_T
    =
    \sum_{t=1}^T\log(1+\lambda_t^\star Z_t)-r_T.
\]
Because $|\lambda_t^\star Z_t|\le1/2$, $\log(1+\lambda_t^\star Z_t)\ge\lambda_t^\star Z_t-(\lambda_t^\star)^2$, and therefore $\log M_T^{\rm ctx}\ge\sum_{t=1}^T\{\lambda_t^\star Z_t-(\lambda_t^\star)^2\}-r_T$.

Define the martingale difference $\Delta S_t:=\lambda_t^\star Z_t-\mathbb E_Q[\lambda_t^\star Z_t\mid\mathcal H_{t-1}^\phi]$ and set $S_T:=\sum_{t=1}^T\Delta S_t$. 
The sequence
$(\Delta S_t)_{t\ge1}$ is a martingale-difference sequence with respect to
$\mathbb H^\phi$, since
\(
    \mathbb E_Q[\Delta S_t\mid\mathcal H_{t-1}^\phi]
    =
    \mathbb E_Q[
        \lambda_t^\star Z_t
        -
        \mathbb E_Q[\lambda_t^\star Z_t\mid\mathcal H_{t-1}^\phi]
        \mid
        \mathcal H_{t-1}^\phi
    ]
    =
    0.
\)
Since $\lambda_t^\star$ is $\mathcal H_{t-1}^\phi$-measurable, $\mathbb E_Q[\lambda_t^\star Z_t\mid\mathcal H_{t-1}^\phi]=\lambda_t^\star m_t^\phi$. Hence
\[
    \sum_{t=1}^T\{\lambda_t^\star Z_t-(\lambda_t^\star)^2\}
    =
    S_T+\sum_{t=1}^T\{\lambda_t^\star m_t^\phi-(\lambda_t^\star)^2\}
    =
    S_T+\sum_{t=1}^T b_t^Q(\theta^\star).
\]
Thus $\log M_T^{\rm ctx}\ge S_T+\sum_{t=1}^T b_t^Q(\theta^\star)-r_T\ge S_T+D_T-r_T$. On the event $\{\tau_\gamma^{\rm ctx}>T\}$, one has $\log M_T^{\rm ctx}<\log(1/\gamma)$ due to monotonicity of the logarithm, and so $S_T<\log(1/\gamma)+r_T-D_T=-s_T(D)$. Therefore $\{\tau_\gamma^{\rm ctx}>T\}\subseteq\{S_T<-s_T(D)\}$.

It remains to control the lower tail of $S_T$.
Conditional on $\mathcal H_{t-1}^\phi$,
$\lambda_t^\star Z_t$ lies between the minimum and maximum of
$-\alpha\lambda_t^\star$ and $(1-\alpha)\lambda_t^\star$.
These endpoints are predictable, and the interval has length
$|\lambda_t^\star|$. Centering by $\lambda_t^\star m_t^\phi$
translates the interval, so $\Delta S_t$ is also contained in a
predictable interval of this length. This remains true when the
conditional hit is deterministic or the stake is zero.
Since $\theta^\star\in K$ and $\|\phi_t\|_2\le R$,
\(
    |\lambda_t^\star|
    \le\|\theta^\star\|_2\|\phi_t\|_2
    \le\frac{1}{2R}R
    =\frac12.
\)
Applying Lemma~\ref{lem:hoeffding-length} with $c_t=1/2$
for every $t$ gives $V_T=\sum_{t=1}^T c_t^2=T/4$ and,
for every $x>0$,
\[
    Q(S_T\le-x)
    \le\exp\!\left(-\frac{2x^2}{V_T}\right)
    =\exp\!\left(-\frac{8x^2}{T}\right).
\]
Taking $x=s_T(D)>0$ proves the result.
\end{proof}

\begin{example}[Intermittent miscalibration]
The cumulative-edge condition allows evidence to accumulate even when
miscalibrated rounds are interspersed with calibrated ones. To illustrate
this, take $d=1$, $\phi_t=R$, and the past-hit filtration
$\mathcal H_t^\phi=\sigma(B_1,\ldots,B_t)$. Let $Q$ give independent hits
with probability $\alpha$ on odd rounds and $\alpha+\delta$ on even
rounds, where $-\alpha<\delta<1-\alpha$ and $\delta\ne0$.
Assume that this law $Q$ belongs to the
ambient model class $\mathcal P$.
Thus, odd rounds are calibrated, while even rounds have a fixed bias:
$\delta>0$ means too many hits and $\delta<0$ means too few.

The average bias over each pair of rounds is $\delta/2$. Choose the
fixed comparator $\theta^\star=\delta/(4R)\in K$, whose constant stake
$\lambda_t^\star=\delta/4$ follows the direction of this bias.
It is feasible because $|\delta|<1$.
On odd rounds, the conditional bias is zero, so the quadratic penalty
gives $b_t^Q(\theta^\star)=-\delta^2/16$.
On even rounds, the conditional bias is $\delta$, giving
$b_t^Q(\theta^\star)=3\delta^2/16$.
Consequently, each pair contributes a positive total edge
$\delta^2/8$: the gain on the miscalibrated round outweighs the
penalty on the calibrated round.

Summing these contributions gives
$B_T^Q(\theta^\star)=\delta^2(4\lfloor T/2\rfloor-T)/16$.
Hence, for every integer $T\ge1$,
$B_T^Q(\theta^\star)\ge D_T:=\delta^2(T-2)/16$,
with equality at odd horizons.
This deterministic lower envelope may be negative initially, but
grows linearly with positive slope.
In contrast, no comparator can have a positive uniform edge:
on every odd round, $b_t^Q(\theta)=-(\theta R)^2\le0$.

Under the learner and regret assumptions of
Theorem~\ref{thm:full-feature-regret-to-power}, sublinear regret
$r_T=o(T)$ gives
$s_T(D)/T=(D_T-r_T-\log(1/\gamma))/T\to\delta^2/16>0$.
The theorem's tail bound therefore implies
$Q(\tau_\gamma^{\rm ctx}<\infty)=1$.
The same fixed comparator is used throughout: detection follows
because its accumulated edge eventually exceeds both the regret
cost and the evidence threshold, despite the calibrated rounds.
\end{example}

\begin{remark}[Lower envelopes as non-i.i.d.~drift conditions]
\label{rem:lower-envelope-noniid-drift}
The lower envelope $D_T$ in Definition~\ref{def:cumulative-full-feature-edge-alternative} is a non-i.i.d.~analogue of a positive drift condition. To make the comparison precise, recall the simple i.i.d.~likelihood-ratio setting: let $X_1,X_2,\ldots$ be i.i.d.~under an alternative law $Q_1$, and let $P_0$ be a simple null. If $\ell(X):=\log(dQ_1/dP_0)(X)$, then under $Q_1$, $\mathbb E_{Q_1}[\ell(X_1)]={\rm KL}(Q_1\Vert P_0)$, where ${\rm KL}$ denotes the Kullback--Leibler information \citep{KullbackLeibler1951Information}. Hence $\mathbb E_{Q_1}[\sum_{t=1}^T\ell(X_t)]=T\,{\rm KL}(Q_1\Vert P_0)$. Thus, in the i.i.d.~likelihood-ratio problem, cumulative evidence has a linear expected drift with per-observation rate ${\rm KL}(Q_1\Vert P_0)$. This is the classical signal rate underlying sequential likelihood-ratio testing \citep{Wald1945Sequential} and the i.i.d.~Chernoff--Stein Lemma \citep[Theorem~11.8.3]{CoverThomas2006Elements}.
In the present problem, there need not be a repeated observation law, a stationary distribution, or independent increments. The one-step predictable edge $b_t^Q(\theta)=\langle\theta,\phi_t\rangle m_t^\phi-\langle\theta,\phi_t\rangle^2$ is $\mathcal H_{t-1}^\phi$-measurable and may vary with time, covariates, and the realized past. Consequently, there is generally no single population drift parameter analogous to ${\rm KL}(Q_1\Vert P_0)$. The deterministic lower envelope
\(
    B_T^Q(\theta^\star)
    :=
    \sum_{t=1}^T b_t^Q(\theta^\star)
    \ge
    D_T
\)
\(
    Q\text{-a.s. for every }T\ge1
\)
records, instead, the cumulative amount of predictable feature-aligned miscalibration available to a fixed comparator $\theta^\star$ along the possibly non-i.i.d. stream.

\end{remark}

\begin{remark}[Connection to linear-full feature alternative in Theorem \ref{thm:power-linear-full-feature}]
\label{remark:connection_to_uniform_edge_cumulative}
For $0<\kappa\le\kappa_{\max}(\alpha)$ and
$t_0\in\mathbb N_{\ge1}$, the uniform-edge class of
Definition~\ref{def:linear-full-feature-alternative} is a subclass
of the cumulative-edge alternative:
\(
    \mathcal Q_{\kappa,t_0}^{\phi,\operatorname{lin}}(\Phi)
    \subseteq
    \mathcal Q_{\rm cum}^{\phi}(D^{\kappa,t_0};\Phi),
    \,
    D_T^{\kappa,t_0}:=\kappa T-2(t_0-1).
\)
The proof of Theorem~\ref{thm:power-linear-full-feature}
verifies this envelope for every integer $T\ge1$,
including the horizons before $t_0$.
In this case, the envelope $D_T^{\kappa,t_0}$ is linear, but the data stream may still be non-i.i.d.
It asserts only linear growth of cumulative predictable edge. Such growth may arise from i.i.d. repetitions, but it may also arise from nonstationary or history-dependent streams, including changepoints and time-varying feature distributions. 
\end{remark}

\subsection{Missing proofs}

\begin{proof}[Proof of Proposition \ref{prop:null-hierarchy}]
Fix $P\in\mathcal P_0(\mathbb G)$. Then, for every $t$,
$\mathbb E_P[Z_t\mid\mathcal G_{t-1}]=0$ $P$-a.s. Since
$\mathcal H_{t-1}\subseteq\mathcal G_{t-1}$, the tower property yields
$\mathbb E_P[Z_t\mid\mathcal H_{t-1}]
=\mathbb E_P[\mathbb E_P[Z_t\mid\mathcal G_{t-1}]\mid\mathcal H_{t-1}]
=0$ $P$-a.s. Hence $P\in\mathcal P_0(\mathbb H)$. Applying this implication
along the filtration chain
$\mathbb H^{\rm marg}\subseteq\mathbb H^{w}\subseteq\mathbb H^{\phi}$
gives the displayed hierarchy.
\end{proof}

\begin{proof}[Proof of Theorem \ref{thm:validity-monitoring}]
Fix $P\in\mathcal P_0(\mathbb H)$. Nonnegativity follows from
$\lambda_t\in\Lambda_\alpha$ and $Z_t\in\{-\alpha,1-\alpha\}$. The process is
adapted to the monitored information because $M_{t-1}$ and $\lambda_t$ are
$\mathcal H_{t-1}$-measurable and $Z_t$ is $\mathcal H_{t}$-measurable. For each finite $t$, $M_t$ is integrable by induction, since each multiplier is
bounded. Using $\mathbb H$-predictability and the definition of
$\mathcal P_0(\mathbb H)$,
\[
\begin{aligned}
    \mathbb E_P[M_t\mid\mathcal H_{t-1}]
    &=
    \mathbb E_P[
        M_{t-1}(1+\lambda_t Z_t)
        \mid
        \mathcal H_{t-1}
    ]                                                    \\
    &=
    M_{t-1}
    \left\{
        1+\lambda_t
        \mathbb E_P[Z_t\mid\mathcal H_{t-1}]
    \right\}                                             \\
    &=
    M_{t-1}.
\end{aligned}
\]
Here, the second equality is because $M_{t-1}$ is $\mathcal H_{t-1}$-measurable. The final equality due to the bounded martingale difference sequence (MDS) property of $(Z_t)_{t\ge1}$ with respect to $\mathbb H$, (Proposition \ref{prop:bounded-monitored-mds}).
Thus $(M_t)$ is a nonnegative $P$-martingale with initial value $M_0=1$.
Since $P$ was arbitrary in $\mathcal P_0(\mathbb H)$, the process is a test
martingale for the composite null.
\end{proof}

\begin{proof}[Proof of Corollary \ref{cor:anytime-valid-monitoring}]
For every $P\in\mathcal P_0(\mathbb H)$,
Theorem~\ref{thm:validity-monitoring} shows that $(M_t)$ is a nonnegative $P$-martingale with $M_0=1$, and hence a nonnegative supermartingale. Ville's inequality for nonnegative supermartingales gives
\[
    P\!\left(\sup_{t\ge0}M_t\ge c\right)
    \le
    \frac{\mathbb E_P[M_0]}{c}
    =
    \frac1c.
\]
Taking $c=1/\gamma$ and using $\{\tau_\gamma<\infty\}=\{\sup_{t\ge1}M_t\ge1/\gamma\}\subseteq\{\sup_{t\ge0}M_t\ge1/\gamma\}$ proves the claim.
\end{proof}

\begin{proof}[Proof of Corollary \ref{cor:validity-transfer}]
Fix any $P\in\mathcal P_0(\mathbb G)$. Since
$\mathcal H_{t-1}\subseteq\mathcal G_{t-1}$, any $\mathcal H_{t-1}$-measurable bet $\lambda_t$ is also
$\mathcal G_{t-1}$-measurable. In other words, any $\mathbb H$-predictable betting process is automatically $\mathbb G$-predictable. Moreover, by definition of
$\mathcal P_0(\mathbb G)$,
\[
    \mathbb E_P[Z_t\mid \mathcal G_{t-1}]=0
    \qquad
    P\text{-a.s.}
\]
The no-bankruptcy constraint gives $1+\lambda_t Z_t\ge0$ because
$Z_t\in\{-\alpha,1-\alpha\}$ and $\lambda_t\in\Lambda_\alpha$. Hence
$M_t\ge0$. Because \(M_{t-1}\) and \(\lambda_t\) are both \(\mathcal{G}_{t-1}\)-measurable, we may pull them out when conditioning on \(\mathcal{G}_{t-1}\).
Therefore,
\[
\begin{aligned}
    \mathbb E_P[M_t\mid\mathcal G_{t-1}]
    &=
    \mathbb E_P[
        M_{t-1}(1+\lambda_t Z_t)
        \mid
        \mathcal G_{t-1}
    ] \\
    &=
    M_{t-1}
    \left(
        1+\lambda_t\mathbb E_P[Z_t\mid\mathcal G_{t-1}]
    \right) \\
    &=
    M_{t-1}.
\end{aligned}
\]
Thus $(M_t)_{t\ge0}$ is a nonnegative $P$-martingale with respect to
$\mathbb G$ for every $P\in\mathcal P_0(\mathbb G)$. Consequently it is a test
martingale, and hence an e-process, for the stronger null
$\mathcal P_0(\mathbb G)$.
Finally, since
\(
    \mathcal H_{t-1}^{\rm marg}
    \subseteq
    \mathcal H_{t-1}^w
    \subseteq
    \mathcal H_{t-1}^\phi,
\)
the preceding argument applies first with
$(\mathbb H,\mathbb G)=(\mathbb H^{\rm marg},\mathbb H^w)$, then with
$(\mathbb H^{\rm marg},\mathbb H^\phi)$, and finally with
$(\mathbb H^w,\mathbb H^\phi)$. This proves the stated transfer of validity.
\end{proof}

\begin{proof}[Proof of Proposition \ref{prop:coarser-tests-can-miss}]
Since $Q\notin\mathcal P_0(\mathbb G)$, the stronger
$\mathbb G$-conditional calibration null fails. Since
$Q\in\mathcal P_0(\mathbb H)$, Theorem~\ref{thm:validity-monitoring} implies
that every $\mathbb H$-predictable no-bankruptcy wealth process is a
nonnegative $Q$-martingale with respect to $\mathbb H$. Ville's inequality then
gives $Q(\tau_\gamma<\infty)\le\gamma$.
It remains to prove the expected log-wealth bound. For each $t$, we have
\[
    \mathbb E_Q[1+\lambda_tZ_t\mid\mathcal H_{t-1}]
    =
    1+\lambda_t\mathbb E_Q[Z_t\mid\mathcal H_{t-1}]
    =
    1.
\]
Since $\log$ is concave, Jensen's inequality gives
\[
    \mathbb E_Q[
        \log(1+\lambda_tZ_t)
        \mid
        \mathcal H_{t-1}
    ]
    \le
    \log
    \mathbb E_Q[
        1+\lambda_tZ_t
        \mid
        \mathcal H_{t-1}
    ]
    =
    0.
\]
Summing from $t=1$ to $T$ gives the following
\[
    \mathbb E_Q[\log M_T]
    =
    \sum_{t=1}^T
    \mathbb E_Q[\log(1+\lambda_tZ_t)]
    \le
    0.
\]
Thus, a coarser skeptic has no positive systematic log-growth under $Q$, even
though $Q$ violates calibration at the richer information level.
\end{proof}

\begin{proof}[Proof of Theorem \ref{thm:power-linear-full-feature}]
Since $Q\in\mathcal Q_{\kappa,t_0}^{\phi,\operatorname{lin}}(\Phi)$, there exists $\theta^\star\in K$ such that $b_t^Q(\theta^\star)\ge\kappa$ $Q$-a.s. for every $t\ge t_0$. For $t<t_0$, write $\lambda_t^\star=\langle\theta^\star,\phi_t\rangle$. Since $|\lambda_t^\star|\le1/2$ and $m_t^\phi\in[-\alpha,1-\alpha]\subseteq[-1,1]$, 
\(
    b_t^Q(\theta^\star)
    =
    \lambda_t^\star m_t^\phi-(\lambda_t^\star)^2
    \ge
    -|\lambda_t^\star|-(\lambda_t^\star)^2
    \ge
    -\frac12-\frac14
    =
    -\frac34
    \ge
    -1.
\)
For every integer $T\ge t_0$,
\[
\begin{aligned}
    \sum_{t=1}^T b_t^Q(\theta^\star)
    &\ge-(t_0-1)+\kappa(T-t_0+1)\\
    &\ge\kappa T-2(t_0-1)=:D_T,
\end{aligned}
\]
where the last inequality uses
$\kappa\le\kappa_{\max}(\alpha)<1/4\le1$.
For $1\le T<t_0$, we have $T\le t_0-1$, so
\(
    D_T=\kappa T-2(t_0-1)
    \le\kappa T-2T
    \le-T
    \le\sum_{t=1}^T b_t^Q(\theta^\star).
\)
Thus $Q\in\mathcal Q_{\rm cum}^{\phi}(D;\Phi)$, with this
linear envelope valid at every integer horizon.
Let $n=n_\gamma^{\rm ctx}(r)<\infty$. For every integer $T\ge n$,
the certification inequality implies
\(
    s_T(D)=D_T-r_T-\log(1/\gamma)\ge\frac{\kappa T}{2}.
\)
Theorem~\ref{thm:full-feature-regret-to-power} therefore gives
\[
    Q(\tau_\gamma^{\rm ctx}>T)
    \le\exp\!\left(-\frac{8(\kappa T/2)^2}{T}\right)
    =\exp(-2\kappa^2T).
\]
Since the events $\{\tau_\gamma^{\rm ctx}>T\}$ decrease to
$\{\tau_\gamma^{\rm ctx}=\infty\}$, this vanishing tail gives
$Q(\tau_\gamma^{\rm ctx}<\infty)=1$.
Write $\tau=\tau_\gamma^{\rm ctx}$. The tail-sum formula,
also valid for an extended expectation by monotone convergence,
gives the following, the penultimate inequality uses $n\ge t_0\ge1$,
and the last uses $e^x-1\ge x$ for $x>0$:
\[
\begin{aligned}
    \mathbb E_Q[\tau]
    &=\sum_{T=0}^{\infty}Q(\tau>T)\\
    &\le n+\sum_{T=n}^{\infty}e^{-2\kappa^2T}\\
    &=n+\frac{e^{-2\kappa^2n}}{1-e^{-2\kappa^2}}\\
    &\le n+\frac{1}{e^{2\kappa^2}-1}\\
    &\le n+\frac{1}{2\kappa^2}.
\end{aligned}
\]
\end{proof}

\subsection{Extension to general quantile calibration}
\label{app:discrete-quantiles}

For discrete outcomes, a correct quantile forecast need not attain
exact hit coverage, because the forecast can coincide with an atom.
This subsection extends our information-aware formulation to the
general quantile null through two directional wealth processes.
We establish anytime validity, explain the information hierarchy
and the effect of learning regret, and illustrate the construction
using correct discrete oracle forecasts.
Fix $\alpha\in(0,1)$. Define $\mathcal P_0^{\mathrm{qt}}(\mathbb H)$
as the class of laws for which, $P$-a.s. for every $t\ge1$,
\(
P(Y_t<\widehat q_t\mid\mathcal H_{t-1})
\le\alpha\le P(Y_t\le\widehat q_t\mid\mathcal H_{t-1}).
\)
Record both hit indicators in $\mathcal H_t$ and set
\(
S_t^{\mathrm{hi}}=\mathbf1\{Y_t<\widehat q_t\}-\alpha,
\qquad
S_t^{\mathrm{lo}}=\alpha-\mathbf1\{Y_t\le\widehat q_t\}.
\)
For predictable stakes $0\le\lambda_t^{\mathrm{hi}}\le1/\alpha$ and
$0\le\lambda_t^{\mathrm{lo}}\le1/(1-\alpha)$, define
$M_t^s=\prod_{i=1}^t(1+\lambda_i^sS_i^s)$, with $M_0^s=1$, and
$M_t^{\mathrm{qt}}=(M_t^{\mathrm{hi}}+M_t^{\mathrm{lo}})/2$.

\begin{proposition}[Anytime validity]
\label{prop:general-quantile-validity}
Both components and their mean are nonnegative test supermartingales
for $\mathcal P_0^{\mathrm{qt}}(\mathbb H)$. Hence, for $\gamma\in(0,1)$,
$\sup_{P\in\mathcal P_0^{\mathrm{qt}}(\mathbb H)}
P(\sup_{t\ge0}M_t^{\mathrm{qt}}\ge1/\gamma)\le\gamma$.
\end{proposition}

\begin{proof}
The null gives $\mathbb E_P[S_t^s\mid\mathcal H_{t-1}]\le0$.
The stake bounds ensure nonnegative, bounded factors, so the products
are adapted and integrable. Predictability gives
$\mathbb E_P[M_t^s\mid\mathcal H_{t-1}]
=M_{t-1}^s(1+\lambda_t^s
\mathbb E_P[S_t^s\mid\mathcal H_{t-1}])\le M_{t-1}^s$.
Linearity gives the same inequality for the mean; Ville's inequality
gives the bound. No continuity or independence is required.
\end{proof}

\paragraph{Information and learning.}
The information hierarchy is preserved:
$\mathcal H_{t-1}\subseteq\mathcal G_{t-1}$ for all $t$ implies
$\mathcal P_0^{\mathrm{qt}}(\mathbb G)\subseteq
\mathcal P_0^{\mathrm{qt}}(\mathbb H)$, by the tower property.
Choose directional comparator classes $K_s$ with predictable
feature-aware stakes $\lambda_t^s(\theta)\in[0,1/2]$, and predictable
learner stakes $\lambda_t^s=\lambda_t^s(\theta_t^s)$,
where $\theta_t^s\in K_s$.
The losses $g_t^s(\theta)=-\log(1+\lambda_t^s(\theta)S_t^s)$ are finite.
If these choices satisfy
$\sum_{t=1}^T\{g_t^s(\theta_t^s)-g_t^s(\theta)\}\le r_T^s$
pathwise for every $\theta\in K_s$, with deterministic $r_T^s$, then
\(
\log M_T^{\mathrm{qt}}\ge
\sum_{t=1}^T\log(1+\lambda_t^s(\theta)S_t^s)-r_T^s-\log2.
\)
This follows from regret and $M_T^{\mathrm{qt}}\ge M_T^s/2$.
Comparator log wealth at least $r_T^s+\log(2/\gamma)$ ensures rejection
by $T$, using regret and growth conditions for these directional classes.

\paragraph{Discrete-oracle illustration.}
We compare the exact-hit audit, the directional wealth average,
and the randomized-hit oracle on correct discrete quantile forecasts
to examine how the choice of hit construction affects anytime
rejection rates.
Figure~\ref{fig:discrete-quantile-validity} uses the Negative-Binomial
design of Appendix~\ref{app:simulation_protocol}, with intercepts
$6$ and $\log(0.75)$, $1{,}000$ series, $T=1{,}460$,
and $\alpha=0.01,\ldots,0.99$.
Bets use past-only plug-in Kelly estimates with two pseudo-observations
at the nominal rate, clipped to $[-0.5,0.5]$ for signed hits and
$[0,0.5]$ for directional scores; the wealth threshold is $20$.
Specifically, for each binary hit sequence $B_i^s$ with nominal
probability $a_s$, we use
$\widehat p_{t-1}^s=(2a_s+\sum_{i<t}B_i^s)/(t+1)$.
This gives the initial estimate $a_s$ the weight of two observations.
We substitute this estimate into the Kelly stake
$\widehat\lambda_t^s=(\widehat p_{t-1}^s-a_s)/\{a_s(1-a_s)\}$
and clip as above.
For the exact-hit audit, $B_i^{\mathrm{ex}}=B_i$ and
$a_{\mathrm{ex}}=\alpha$.
The directional components use
$B_i^{\mathrm{hi}}=\mathbf1\{Y_i<\widehat q_i\}$ with
$a_{\mathrm{hi}}=\alpha$, and
$B_i^{\mathrm{lo}}=1-B_i$ with $a_{\mathrm{lo}}=1-\alpha$.
The randomized-hit oracle (Figure~\ref{fig:discrete-quantile-validity}) uses the true conditional CDF
$F_t(y)=P(Y_t\le y\mid\mathcal F_{t-1})$.
With independent draws $V_t\sim\operatorname{Uniform}(0,1)$,
independent of the data, define
$U_t=F_t(Y_t-)+V_t\{F_t(Y_t)-F_t(Y_t-)\}$ and record
$B_t^{\mathrm{rand}}=\mathbf1\{U_t\le\alpha\}$.
This spreads observations uniformly across their CDF jumps.
Under the correct conditional CDF,
$P(B_t^{\mathrm{rand}}=1\mid\mathcal H_{t-1})=\alpha$,
including for discrete outcomes.
The benchmark uses these hits in the Kelly rule above with
$a_{\mathrm{rand}}=\alpha$.
It requires the full conditional CDF, whereas the directional
construction uses only point forecasts and the recorded hit indicators. We can observe that the directional average is valid in both regimes, while the exact-hit test loses validity for the low-count (intermittent) regime. All the rest of our main experiments concern the high-count regime for simulated as well as real data, allowing us to focus on the exact-hit null and the corresponding test without extra complications of the two-sided extension.

\begin{figure}[!ht]
\centering
\includegraphics[width=0.95\linewidth]{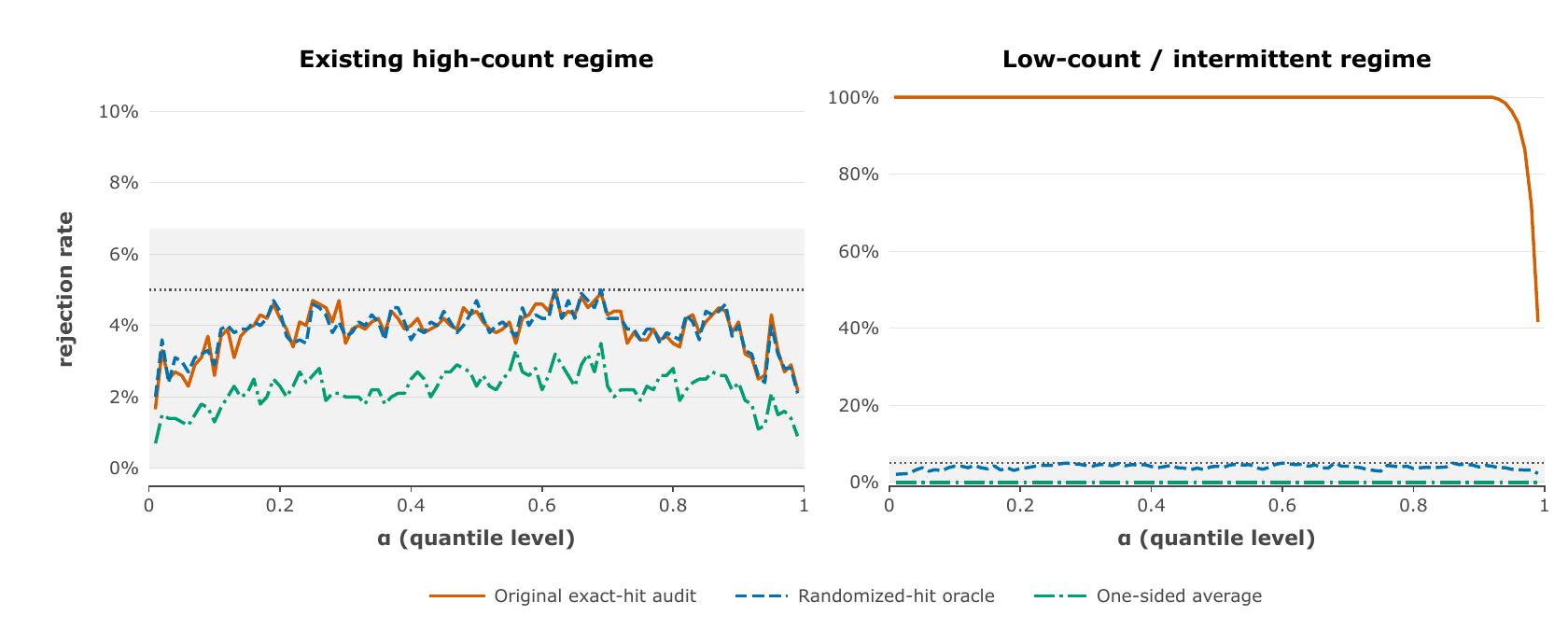}
\caption{Oracle-quantile rejection rates. Exact-hit equality can fail
at atoms while the quantile inequalities hold. Randomized hits use
the full oracle CDF. Dotted line: $5\%$; shading: pointwise $99\%$
binomial tolerance up to $6.7\%$, without simultaneous coverage.
The directional average is valid in both regimes, while the exact-hit test loses validity for the low-count (intermittent) regime.}
\label{fig:discrete-quantile-validity}
\end{figure}

\section{Experimental details}
\label{app:experiments}

\paragraph{Code.}
The anonymized source code for all experiments, with its dependency
specification, is available at
\url{https://anonymous.4open.science/r/Bet-on-Features-119C/README.md}.

\subsection{Simulation protocol}
\label{app:simulation_protocol}

\paragraph{Synthetic Negative-Binomial panel.}
The validity experiment and the synthetic-data half of the real-data experiment (both in Section~\ref{sec:experiments}) share a panel of $N{=}1{,}000$ count series over $T{=}1{,}460$ days. For each series and day $t$, demand follows
\[
D_t \sim \mathrm{NegBin}(n_t, p_t), \qquad
\mu_t = \exp\bigl(x_t^\top \beta\bigr), \qquad
p_t = \frac{n_t}{n_t + \mu_t},
\]
with eight-dimensional covariate vector
\[
x_t = \bigl(1,\ \sin^{(w)}_t,\ \cos^{(w)}_t,\ \mathrm{promo}_t,\ \log\!\mathrm{price}_t,\ \sin^{(y)}_t,\ \cos^{(y)}_t,\ g_t\bigr).
\]
The components are:
\begin{itemize}
  \item Weekly seasonality $\sin^{(w)}_t = \sin(2\pi t/7)$, $\cos^{(w)}_t = \cos(2\pi t/7)$ and yearly seasonality $\sin^{(y)}_t = \sin(2\pi t/365)$, $\cos^{(y)}_t = \cos(2\pi t/365)$.
  \item A 2-state Markov promotion flag $\mathrm{promo}_t \in \{0,1\}$ with $\Pr(\mathrm{promo}_t{=}1 \mid \mathrm{promo}_{t-1}{=}0) = 0.06$ and $\Pr(\mathrm{promo}_t{=}1 \mid \mathrm{promo}_{t-1}{=}1) = 0.70$.
  \item A random-walk log-price $\log\!\mathrm{price}_t = \log\!\mathrm{price}_{t-1} + 0.02\,\xi_t$ plus an iid jitter $0.10\,\eta_t$, with $\xi_t, \eta_t \sim \mathcal{N}(0,1)$.
  \item An AR(1) latent factor $g_t = 0.95\, g_{t-1} + 0.10\, \zeta_t$, $\zeta_t \sim \mathcal{N}(0,1)$.
\end{itemize}
The coefficient vector is $\beta = (6,\ 0.20,\ -0.08,\ 0.35,\ -0.45,\ 0.08,\ -0.08,\ 0.15)$, and the NB dispersion is $n_t = 10$ on promotion days and $n_t = 15$ otherwise. Per-series random seeds are drawn deterministically from a SHA-256 hash of the tag \texttt{("cov", T, series\_id)}, so the panel is fully reproducible.

\paragraph{Oracle quantiles.}
The oracle in the validity experiment of Section~\ref{sec:experiments} issues the analytic Negative-Binomial $\alpha$-quantile $\hat q_t = F^{-1}_{\mathrm{NB}(n_t, p_t)}(\alpha)$ at $Q{=}99$ levels $\alpha \in \{0.01, 0.02, \ldots, 0.99\}$ per series. Since the conditional law is exactly $\mathrm{NegBin}(n_t, p_t)$, the oracle prediction matches the true conditional quantile and any rejection by the betting test counts as a Type-I error by design.

\paragraph{Black-box forecasters.}
Both forecasters are used zero-shot in float32, with no fine-tuning and
no post-processing of their quantiles. Chronos-2 \citep{ansari2025chronos} is
the single public checkpoint \texttt{amazon/chronos-2} ($120$M parameters),
run through \texttt{chronos-forecasting}~2.2.2. It emits quantiles directly
from a multi-quantile output head, without sampling; we request a prediction
length of one at its $21$ native levels
$\alpha\in\{0.01,0.05,0.10,0.15,\ldots,0.90,0.95,0.99\}$ and disable
cross-series in-context learning, so each series is forecast on its own.
Moirai-2 \citep{liu2025moirai} is the public checkpoint
\texttt{Salesforce/moirai-2.0-R-small} ($11.4$M parameters), run through
\texttt{uni2ts}~2.0.0 with prediction length one and the context length set to
the window; it returns its $9$ native levels $\alpha\in\{0.1,\ldots,0.9\}$,
again without sampling. The betting audits run at every native level, which
is why the Moirai-2 panels show nine points. Forecasts are one-step-ahead
with a rolling origin: for each series and day $t$ the model receives the
target and the covariates over a sliding context of the last $W$ observations
($W{=}365$ in the main text, $W\in\{30,90,180\}$ in
Appendix~\ref{app:additional_simulations}) together with the covariates of
day $t$ as known future covariates, and the first origin follows $365$
observations of history. On the synthetic panel the covariates are the seven
non-constant components of $x_t$, including $g_t$, for the first $500$ series,
giving $1{,}095$ forecasts per series (days $366$ to $1{,}460$). On Rossmann
each series is the sequence of open days of one of the $1{,}115$ stores
between 29 January 2013 and 31 July 2015, the target is daily sales, and the
$13$ covariates are promotion, school holiday, Saturday, the day-of-week sine
and cosine, and sales and customer counts lagged $7$, $14$, $21$ and $28$
calendar days; stores have between $182$ and $549$ forecast days (median
$391$). The covariate features of the skeptics in
Section~\ref{sec:experiments} are exactly these columns, so every covariate a
skeptic bets on was available to the forecaster.
Chronos-2 and the \texttt{chronos-forecasting} package are released under
the Apache-2.0 license, the Moirai-2.0-R-small weights under the
non-commercial Creative Commons BY-NC 4.0 license and the \texttt{uni2ts}
package under Apache-2.0, and the River package used by the skeptics below
under the BSD 3-Clause license. The Rossmann data are the public data of the
2015 Kaggle competition \citep{cukierski2015rossmann}, used under the
competition rules.

\paragraph{Betting configuration.}
All skeptics of Section~\ref{sec:experiments} and
Appendices~\ref{app:sticky_coin} and~\ref{app:additional_simulations} learn
their betting direction with
FTRL (Table~\ref{tab:oco-contextual-bets}; implementation below). The
feature dictionary is the covariate list of the forecaster paragraph above,
seven columns on the synthetic panel and $13$ on Rossmann: the marginal
skeptic bets on an intercept alone, a single-feature skeptic on one covariate
and the intercept, and the full skeptic on all covariates, the intercept and
the previous centered hit $Z_{t-1}$. The marginal
and single-feature skeptics use learning rate $\eta=0.1$ and no $\ell_1$
penalty. The full-feature skeptic runs the four settings $\eta\in\{0.1,0.01\}$,
$\ell_1\in\{0.1,0.01\}$ and reports the equal-weight average of their wealth
processes,
\(
M_t^{\rm full}=\tfrac14\sum_{k=1}^{4}M_t^{(k)},
\)
and ``grand specialists'' is the equal-weight average
$M_t^{\rm grand}=\tfrac1F\sum_{f=1}^{F}M_t^{(f)}$ over the $F$ specialist
(non-full) skeptics of a panel: on the synthetic panel the marginal, the
$Z_{t-1}$-only and the seven single-feature skeptics ($F{=}9$); on Rossmann
the marginal and $Z_{t-1}$-only skeptics, the $13$ single-feature skeptics and
three group skeptics (calendar, customer lags, sales lags, each with
$Z_{t-1}$), $F{=}18$. An average of test martingales with constant
weights is a test martingale \citep{vovk2021evalues,waudbysmith2024testing}, so
Theorem~\ref{thm:validity-monitoring} and the threshold $1/\gamma_{\rm test}$
apply to $M^{\rm full}$ straightforwardly. 
{As for the grand specialist $M^{\rm grand}$, it is a test martingale for the intersection null $\cap_w \mathcal{P}_0^w$, which is a strict superset of the full-feature null $\mathcal{P}_0^\phi$ tested by $M^{\rm full}$, and it is also adapted to the full-feature filtration \(\mathbb H^\phi\). This makes $M^{\rm grand}$ a (theoretically) conservative comparator to $M^{\rm full}$.}

\paragraph{Skeptic implementation.}
The learner is FTRL-Proximal \citep{mcmahan2013adclick}, the
per-coordinate adaptive form of the FTRL update in
Table~\ref{tab:oco-contextual-bets}, as implemented in the River
streaming-learning package~\citep{montiel2021river} with
$\texttt{alpha}=\eta$, $\texttt{beta}=1$, $\texttt{l2}=1$ and \texttt{l1} as
stated above. Covariates are standardized online, with running mean and
variance updated after each bet; the bet is
$\lambda_t=\operatorname{clip}_{\Lambda_\alpha}(\langle\theta_t,\phi_t\rangle)$
with $\theta_1=0$, and $\theta_t$ is updated with the gradient of
$g_t(\theta)=-\log(1+\langle\theta,\phi_t\rangle Z_t)$ in rounds where the
clip is inactive.

\subsection{Sticky-coin experiment: marginal versus conditional null}
\label{app:sticky_coin}

This experiment removes the forecasting layer and feeds the skeptic a controlled hit
process. Specifically, we use the two-state Markov chain of
\citet{christoffersen2004backtesting}, reparametrized so that the hit indicator
$B_t=\mathbf 1\{Y_t\le \hat q_t\}$ has marginal probability $\alpha$ and
one-step autocorrelation $\delta$. In this experiment, the pair
$(Y_t,\hat q_t)$ is not modeled separately; instead, the centered hit
$Z_t=B_t-\alpha$ is generated directly from $(\alpha,\delta)$. This isolates
the effect of monitoring information on power.
For each cell of the $(\alpha,\delta)$ grid, we run $N$ Monte Carlo trials of
length $T$ and average the rejection time $\bar\tau$.
The marginal skeptic bets on the intercept alone and the $Z_{t-1}$-aware
skeptic on the intercept and $Z_{t-1}$, both with the FTRL learner of
Appendix~\ref{app:simulation_protocol} at learning rate $0.1$ and no $\ell_1$
penalty; this grid uses $\gamma_{\rm test}=0.01$, a wealth threshold of $100$,
in place of the $0.05$ of Section~\ref{sec:experiments}.
Figures~\ref{fig:sticky_zprev} and~\ref{fig:sticky_heatmap} show
the gap predicted by Proposition~\ref{prop:coarser-tests-can-miss}: the
intercept-only marginal skeptic is censored at $T$ across the valid region,
because it cannot exploit one-step dependence, whereas the $Z_{t-1}$-aware
skeptic rejects quickly across most of the grid.

\begin{figure}[!ht]
    \centering
    \includegraphics[width=0.8\textwidth]{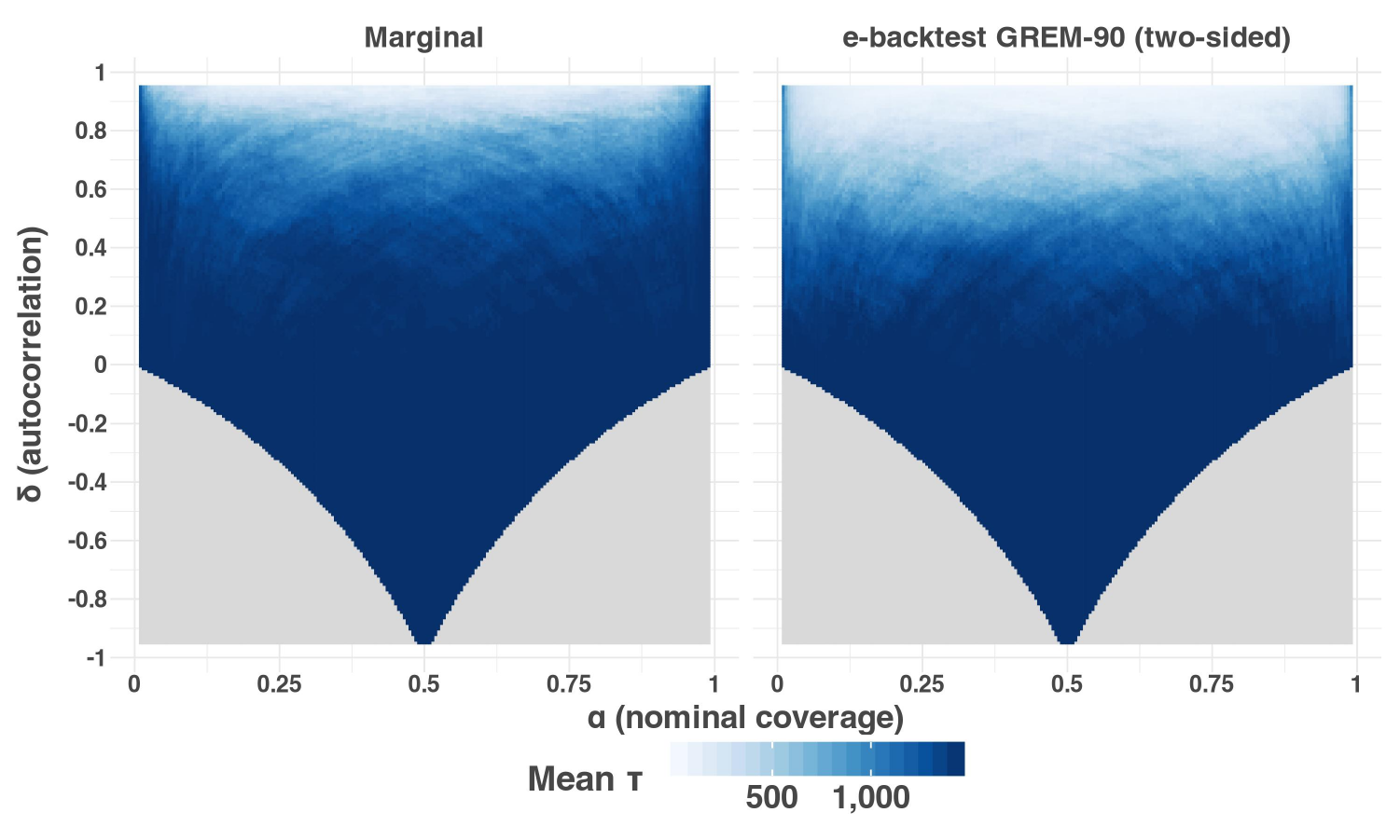}
    \caption{Mean rejection time $\bar\tau$ over the $(\alpha,\delta)$ grid
    for the marginal FTRL skeptic (left) and the two-sided GREM-90 e-backtest
    (right, Appendix~\ref{app:ebacktesting-comparison})
    ($N{=}50$ runs, $T{=}1{,}460$,
    $\gamma_{\rm test}{=}0.01$, the one figure not at the $0.05$ level
    used elsewhere). Censored cells, where no rejection occurs by
    time $T$, appear dark blue; faster rejections appear light blue; grey marks
    invalid parameter pairs. The marginal skeptic is censored at or near $T$ across most of the
    valid region and the feature-agnostic e-backtesting rule tracks it; the
    $Z_{t-1}$-aware skeptic of Figure~\ref{fig:sticky_zprev} rejects quickly
    on the same grid.}
    \label{fig:sticky_heatmap}
\end{figure}

\subsection{Additional simulation results}
\label{app:additional_simulations}

\paragraph{Per-feature rejection across context windows (Chronos-2).}
We complement Figure~\ref{fig:chronos_panel} (Chronos-2 with
a 365-day sliding window in the main body) with the same combined
Rossmann-Synthetic panels for the remaining context windows.
All panels use the betting configuration of
Appendix~\ref{app:simulation_protocol}. Every single-feature skeptic is drawn;
promotion, Saturday and log-price are highlighted in all panels as the
strongest signals, the remaining single-feature skeptics appear in faded
grey. The e-backtesting curve in every panel is the two-sided GREM-90 rule
of Appendix~\ref{app:ebacktesting-comparison}. See
Figure~\ref{fig:chronos_panel} for the panel layout.

\begin{figure}[h!]
    \centering
    \includegraphics[width=0.95\textwidth]{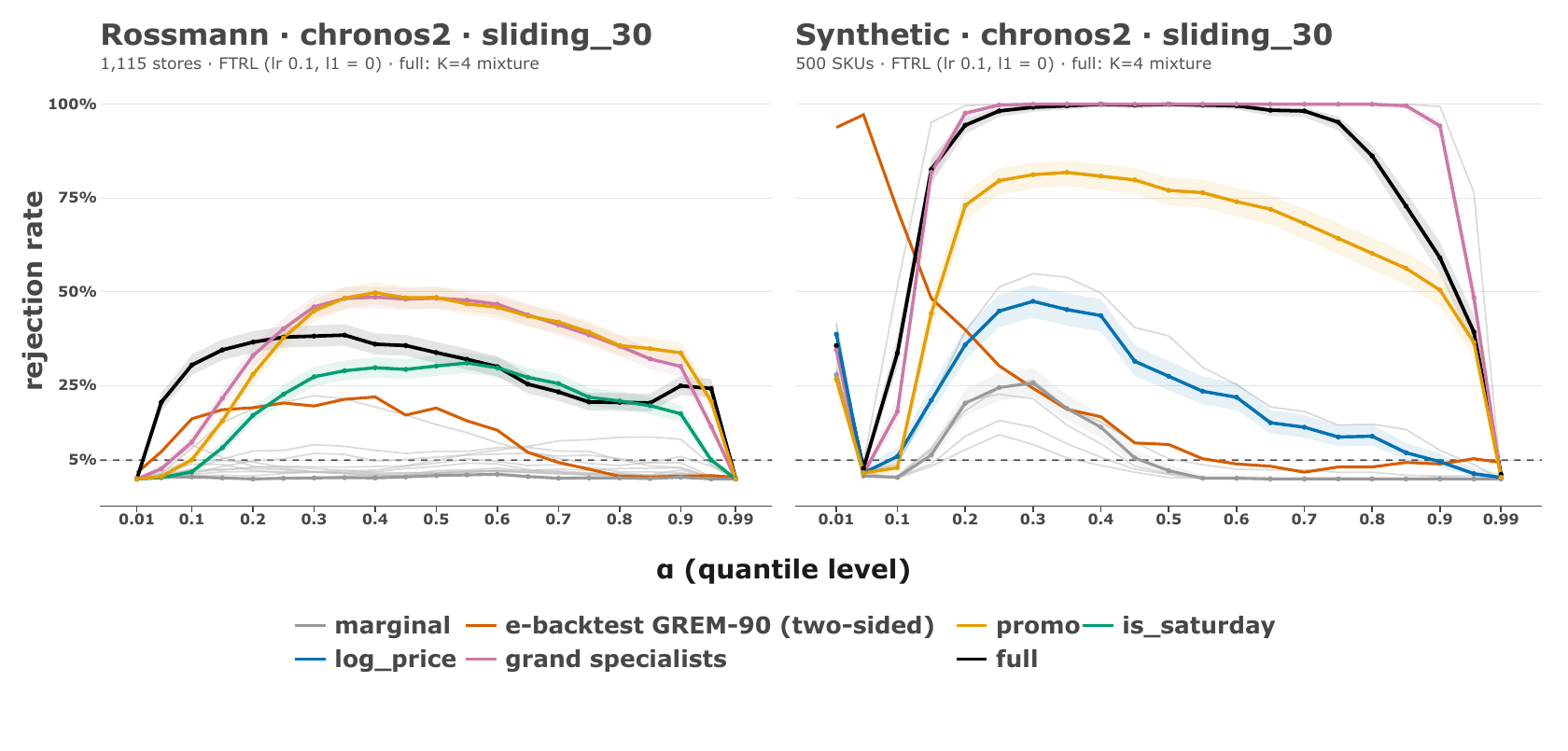}
    \caption{Per-feature rejection rate vs.\ quantile~$\alpha$ for
    Chronos-2 with a 30-day sliding window (FTRL betting with the
    settings of Figure~\ref{fig:chronos_panel}). Companion to
    Figure~\ref{fig:chronos_panel}.}
    \label{fig:appx-panel-chronos2-30}
\end{figure}

\begin{figure}[h!]
    \centering
    \includegraphics[width=0.95\textwidth]{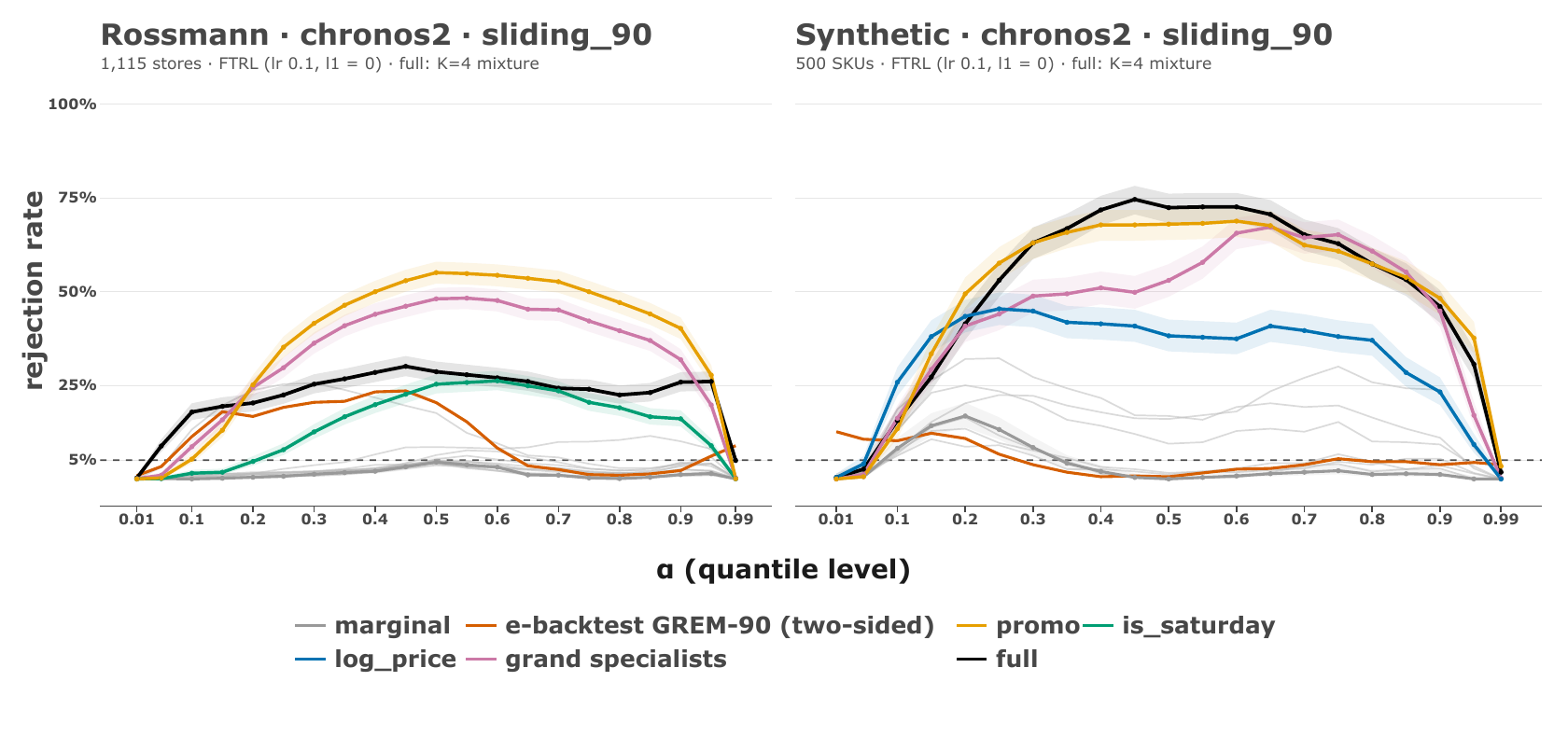}
    \caption{Per-feature rejection rate vs.\ quantile~$\alpha$ for
    Chronos-2 with a 90-day sliding window (FTRL betting with the
    settings of Figure~\ref{fig:chronos_panel}). Companion to
    Figure~\ref{fig:chronos_panel}.}
    \label{fig:appx-panel-chronos2-90}
\end{figure}

\begin{figure}[h!]
    \centering
    \includegraphics[width=0.95\textwidth]{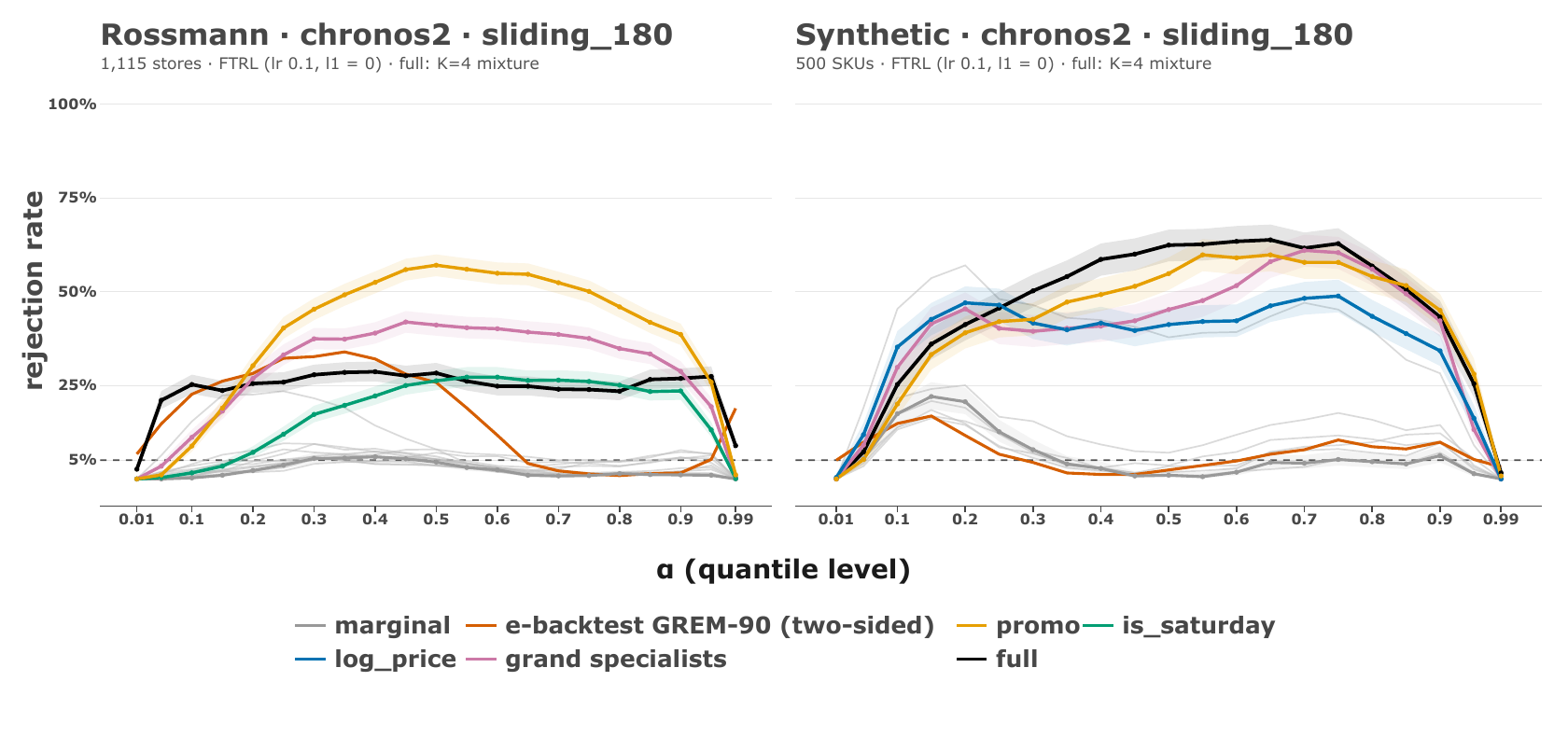}
    \caption{Per-feature rejection rate vs.\ quantile~$\alpha$ for
    Chronos-2 with a 180-day sliding window (FTRL betting with the
    settings of Figure~\ref{fig:chronos_panel}). Companion to
    Figure~\ref{fig:chronos_panel}.}
    \label{fig:appx-panel-chronos2-180}
\end{figure}

\paragraph{Per-feature rejection: Moirai-2 forecaster.}
The same panel for the second foundation forecaster (Moirai-2) at the
matching 365-day sliding window. Moirai-2 attains substantially higher
rejection rates across all features, indicating heavier conditional
miscalibration; nevertheless the per-feature ordering remains
informative for diagnosis.

\begin{figure}[h!]
    \centering
    \includegraphics[width=0.95\textwidth]{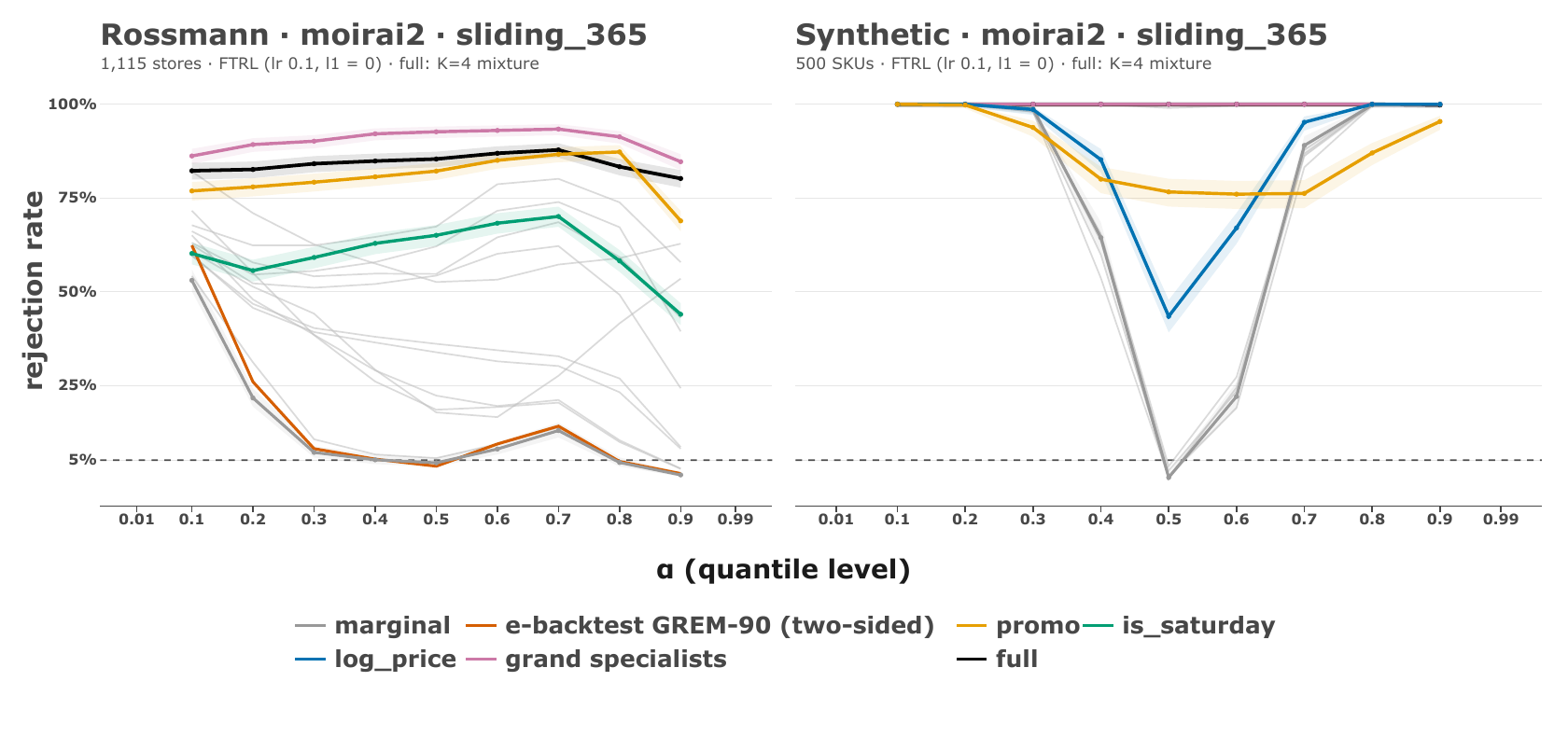}
    \caption{Per-feature rejection rate vs.\ quantile~$\alpha$ for
    Moirai-2 with a 365-day sliding window (FTRL betting with the
    settings of Figure~\ref{fig:chronos_panel}). Companion to
    Figure~\ref{fig:chronos_panel}. On the synthetic panel the
    two-sided e-backtest rejects at every level and the marginal skeptic at
    most levels: Moirai-2 is marginally miscalibrated there.}
    \label{fig:appx-panel-moirai2-365}
\end{figure}

\clearpage

\subsection{Computational considerations}
\label{app:computational}

Forecaster quantile predictions (Chronos-2, Moirai-2) in the
configuration of Appendix~\ref{app:simulation_protocol} were generated on
a single workstation with one NVIDIA GeForce RTX 4090 GPU; predictions
are cached as Parquet files and reused across all betting experiments.
Betting tests, validity runs, and Monte Carlo heatmaps are CPU-only
and ran on an internal Slurm cluster as single-core tasks, one per cell,
skeptic and setting. A task takes minutes; the betting runs behind
Figures~\ref{fig:validity} and~\ref{fig:chronos_panel} and
Appendix~\ref{app:additional_simulations} total about $60$ CPU-hours, and
the sticky-coin grid of Figures~\ref{fig:sticky_zprev} and~\ref{fig:sticky_heatmap} took $7$ CPU-hours
on 64 cores. Development runs with other learners and settings, not
reported here, consumed several hundred CPU-hours.

\end{document}